%% file: main.tex
\pdfoutput=1

\documentclass[11pt]{article}
\usepackage[letterpaper,margin=1in]{geometry}

\usepackage[round]{natbib}

\usepackage[hyphens]{url}
\usepackage{graphicx}
\usepackage{adjustbox}
\usepackage{booktabs}
\usepackage{array}
\usepackage{longtable}
\usepackage{placeins}
\usepackage{amsmath, amssymb, amsthm, mathtools}
\usepackage{enumitem}
\usepackage{xcolor}
\usepackage[colorlinks=true,linkcolor=blue,citecolor=blue,urlcolor=blue]{hyperref}
\usepackage[capitalise,nameinlink]{cleveref}
\usepackage{algorithm}
\usepackage{algpseudocode}
\crefname{algorithm}{Algorithm}{Algorithms}

\newtheorem{theorem}{Theorem}
\newtheorem{proposition}{Proposition}
\newtheorem{corollary}{Corollary}
\newtheorem{lemma}{Lemma}
\theoremstyle{definition}

\newtheorem{assumption}{Assumption}
\theoremstyle{remark}

\theoremstyle{plain}
\crefname{assumption}{Assumption}{Assumptions}
\crefname{lemma}{Lemma}{Lemmas}

\newcommand{\EE}{\mathbb{E}}
\newcommand{\PP}{\mathbb{P}}
\newcommand{\RR}{\mathbb{R}}
\newcommand{\VV}{\operatorname{Var}}
\newcommand{\Cov}{\operatorname{cov}}

\newcommand{\cN}{\mathcal{N}}

\newcommand{\range}{\operatorname{range}}

\newcommand{\st}{\mathsf s}   %
\newcommand{\cm}{\mathsf c}   %
\newcommand{\ec}{\mathsf e}   %

\providecommand{\coloneqq}{\mathrel{:=}}

\title{What Do Sparse Autoencoder Feature Discovery Scores Identify?\\ From Activations to Interventions}
\author{Yifeng Guo$^{1}$\quad Yanran Li$^{2}$\quad Xiang Chen$^{1}$\\[4pt]
{\normalsize $^{1}$St.\ Jude Children's Research Hospital\qquad $^{2}$Columbia University}}
\date{}

\begin{document}

\maketitle

\begin{abstract}
Sparse autoencoder (SAE) features are typically discovered from how their activations co-vary with behaviour, yet they are often used through interventions such as steering or ablation. We take intervention attribution, the first-order effect of a specified intervention on the model output, as our estimand. For a broad class of response functions, we show that observational feature-discovery evidence generally does not identify this attribution, and derive an exact gradient-based decomposition of the gap between linear-probe predictions and intervention attribution. Across four model-SAE pairs, median probe-to-attribution magnitude ratios range from tens to hundreds, while attributions closely track small-step steering. For single-latent ablations with nonzero effects, gradient predictions match the realized sign in 91–97\% of cases, compared with 17–35\% for probe-implied changes. Attribution-selected sets also produce larger absolute mean ablation effects on held-out prompts than sets selected by correlation, rank tests or probe coefficients. Together, these results establish intervention attribution as a principled, gradient-estimable target for selecting SAE features for intervention.
\end{abstract}

\section{Introduction}\label{sec:intro}

Sparse autoencoders (SAEs) rewrite a language model's activations as sparse combinations of thousands of latent features \citep{bricken2023monosemanticity, cunningham2023sparse,templeton2024scaling}. Practitioners discover the features behind a behaviour observationally. They rank latents by how their activations co-vary with the behaviour across natural prompts, using correlations, class-mean differences or per-latent tests \citep{kantamneni2025sparse,patel2026functional,cho2026corrsteer}. They then use the selected features through interventions that promote or remove a concept, suppress an answer or correct a failure \citep{templeton2024scaling,marks2024sparse,arad2025saes}. Steering adds a multiple of a decoder direction \citep{chalnev2024improving}, ablation zeroes a latent's activation or replaces it with a reference mean \citep{wang2023ioi,heimersheim2024how}, and patching transplants it from another prompt \citep{vig2020investigating}. That is, discovery is observational, yet use is interventional.

\begin{figure*}[t]
\centering
\includegraphics[width=\textwidth]{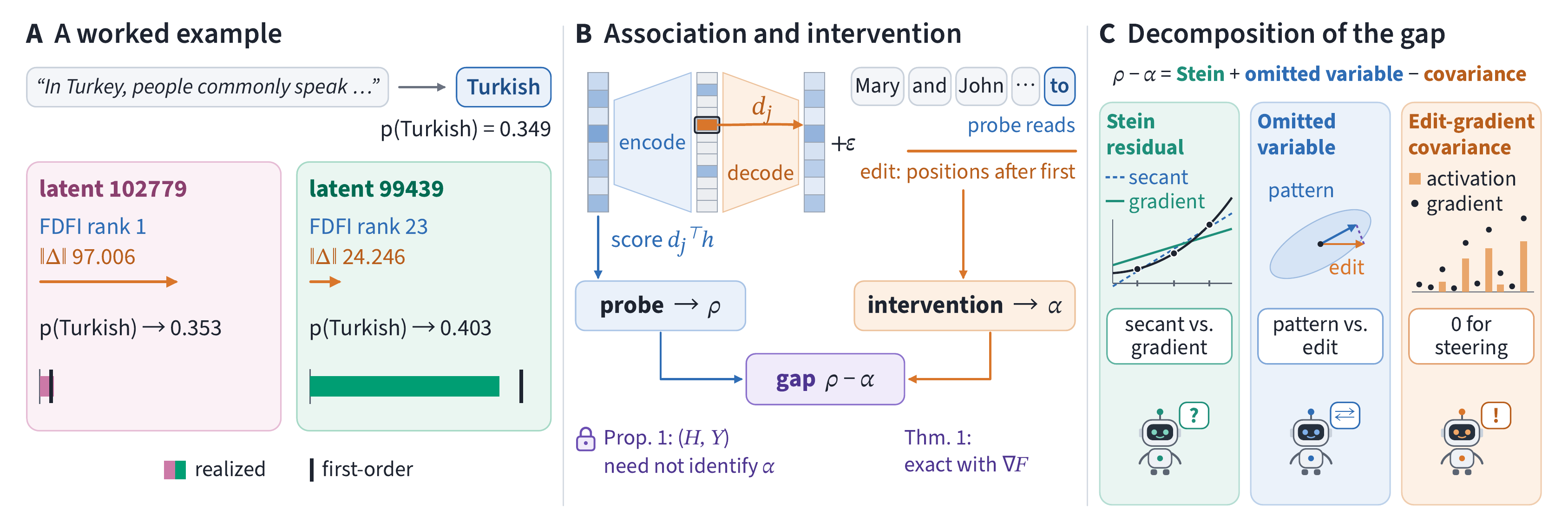}
\caption{Association and intervention for SAE latents. (A) Gemma-2-9B completes ``In Turkey, people commonly speak'' with $p(\text{Turkish})=0.349$. FDFI ranks latent 102779 first and latent 99439 23rd among the $39$ latents it selects for this task. Zero-ablating each latent at every position but the first gives the realized intervention effect (bars) and the prompt-specific gradient prediction (ticks) in standardized log-probability units. The illustrative example is post hoc. It is the sole eligible unit under a selection rule specified after the aggregate analyses with a required positive attribution-set versus FDFI-set contrast (\Cref{app:case}). (B) A separately fitted probe on the final-position score $d_j^\top h$ implies a change $\rho$, whereas an intervention that edits along $d_j$ has first-order effect $\alpha$, its attribution. The law of $(H,Y)$ fixes $F$ only on the support of $H$ and need not identify $\alpha$ (\Cref{prop:discrete}). (C) With gradients the gap $\rho-\alpha$ is a Stein residual plus an omitted-variable term minus an edit-gradient covariance (\Cref{thm:gap}). The secant illustration assumes affinely independent support.}\label{fig:overview}
\end{figure*}

The two can diverge. \Cref{fig:overview}A gives a concrete example using FDFI, a dependence-aware feature-importance method \citep{du2025disentangled,chen2026fdfi}, as an observational discovery pipeline. Zero-ablating the top-ranked latent barely changes the probability of ``Turkish'', despite inducing an edit four times larger in norm, whereas ablating the 23rd-ranked latent raises it from $0.349$ to $0.403$. A single backward pass closely anticipates both changes. The example illustrates the central distinction: a stronger observational discovery signal need not imply a larger intervention effect.

A specified intervention calls for a specified estimand. Write $H$ for the activation, $F(H)$ for a model output such as the standardized log-probability of the answer and $H\mapsto H+\Delta(H)$ for an intervention with edit $\Delta(H)$. Its first-order effect averaged over prompts, $\alpha(\Delta)=\EE[\Delta(H)^\top\nabla F(H)]$, is the \emph{intervention attribution}, or \emph{attribution} for short. Attribution patching computes it with one backward pass per prompt \citep{nanda2023attribution,syed2024attribution}. The estimand is intervention-specific and local. A finite effect also depends on how the gradient changes along the edit path. We therefore test it against realized interventions. Small-step steering checks its local accuracy, larger steps map its range of validity and single-latent ablations give a second check.

Our first question is whether observational discovery evidence identifies \(\alpha(\Delta)\). In general, it does not (\Cref{prop:discrete}). Natural prompts reveal \(F\) only along the activation variation they realize. Two response functions can induce the same observational law while differing in slope along an unexplored direction, leading to different attributions for interventions that move along that direction. Unless the admissible response class rules out this variation, the attribution remains unidentified from observational evidence. This boundary applies to any functional of the observational law, including FDFI scores, correlations, rank tests, class-mean differences and probe coefficients, both at the population level and under the empirical law of the observed prompts. In our evaluation strata, the median fraction of the mean steering edit lying outside the observed activation span is nearly one, making this boundary directly relevant to the finite evaluation design we study.

\Cref{prop:discrete} concerns observational discovery statistics in general. \Cref{thm:gap} then studies a canonical way to translate observational association into an intervention-scale prediction using a linear least-squares probe. A discovery score such as an FDFI statistic ranks latents but does not quantify how much the output should change under a specified edit. Regressing the output on a latent's activation yields such a prediction, which we call the \emph{probe-implied change} \(\rho(\Delta)\). We use this probe as a common observational benchmark for comparing association with intervention attribution. For a fixed scalar activation and target, a correlation determines the probe coefficient after rescaling, as does a class-mean difference for a binary target (\Cref{app:survey}).

Our second question is why $\rho(\Delta)$ and $\alpha(\Delta)$ differ once gradients are computed. Their difference splits exactly into three terms (\Cref{thm:gap}, \Cref{tab:terms}). A \emph{Stein residual} compares the slope that the probe fits across prompts with the model's average local sensitivity. An \emph{omitted-variable term} compares the activation pattern that moves with the probe score with the coordinates that the intervention changes, for instance at another token position. An \emph{edit-gradient covariance} arises when the size of the edit and the model's sensitivity vary together across prompts, as for ablations. For example, a probe that reads only the final position implies no change for zero-ablating a latent that never fires there, whatever its attribution (\Cref{cor:ablate}). One backward pass per prompt makes every term estimable.

Our third question is whether this distinction matters in practice (\Cref{sec:results}). Across four model-SAE pairs, it does in three ways. First, the gap is large. For FDFI-selected directions, the attribution is only a few percent of the probe coefficient in the median. The mismatch extends beyond FDFI-based selection: for latents selected by coefficient magnitude, probe coefficients exceed realized small-step steering responses by more than two orders of magnitude in the model-specific medians. Second, attribution predicts realized interventions. It closely tracks small-step steering and remains more accurate than the probe coefficient at displacements scaled to each latent's maximum observed activation. Prompt-specific gradient predictions also substantially outperform probe-based predictions for single-latent ablations. Third, the distinction affects feature choice. Within FDFI-selected candidates, attribution more often identifies the latent with the largest realized steering response than either the FDFI statistic or the probe coefficient. Attribution-selected sets also produce larger absolute mean ablation effects on held-out prompts than sets selected by correlation, rank tests or probe coefficients.

Our contributions address these three questions.
\begin{itemize}[leftmargin=1.2em,itemsep=0pt,topsep=2pt]
\item \emph{Identification.} We establish an identification boundary showing that observational feature-discovery evidence does not in general identify the first-order effect of a specified intervention (\Cref{sec:identify}).

\item \emph{Diagnosis.} We derive an exact gradient-based decomposition of the gap between the probe-implied change and intervention attribution into three estimable terms, separating regression, omitted-variable and edit-gradient effects (\Cref{sec:gap}).

\item \emph{Intervention consequences.} Across four public model--SAE pairs spanning GPT-2, Gemma 2 and Llama 3.1 and an eleven-task setup, we show that the gap is large in practice. The proposed intervention attribution target predicts realized steering and single-latent ablations more accurately than probe-based predictions and attribution-based selection identifies latents with larger held-out intervention effects (\Cref{sec:results}).
\end{itemize}

\noindent\textbf{Related work.} Attribution patching estimates patching effects from one backward pass \citep{nanda2023attribution,syed2024attribution}, with refinements for attention saturation, integrated gradients and second-order terms \citep{kramar2024atp,hanna2024faith,wang2026attribution}. \citet{marks2024sparse} use it to discover and intervene on sparse feature circuits, whereas we use it as the estimand against which observational evidence is judged. Empirical studies find that activation-side and output-side evidence about latents rarely coincide \citep{arad2025saes,shu2025beyond,tiwari2026decodability,wu2025axbench}. Selection-error control \citep{candes2018panning,enkhbayar2025knockoffs} and ablation-only validation \citep{khera2026anytime} leave the intervention's attribution untested. Probing studies separate what a representation encodes from what a probe learns \citep{hewitt2019designing} or what the model uses \citep{elazar2021amnesic,belinkov2022probing}. Our identity makes this separation exact for linear probes.

\section{From association to interventions}\label{sec:setup}

\noindent\textbf{Strata, activations and outputs.} For each model we partition prompts into \emph{strata} by task, answer and tokenized prompt length and define every estimand within a stratum. Within a stratum the answer and the positions are fixed. Variation in the activations then reflects the prompt rather than a change of answer or length. Pooling over strata would make a probe largely a contrast between answers and lengths, since the between-stratum part accounts for most of a task-pooled probe coefficient in the median (\Cref{app:protocols}). Write $H\in\RR^p$ for the activation at every position after the first and $F(H)$ for the output, the teacher-forced log-probability of the answer standardized within task. With $E_\tau$ the matrix that selects position $\tau$, $h\coloneqq E_{\rm fin}H$ is the final-position block. Within a stratum $H$ has a law with finite second moments, mean $\mu$ and covariance $\Sigma_H$. This law can be the population law of the stratum or the empirical law of its observed prompts. \Cref{prop:discrete,thm:gap} apply to either law under their respective hypotheses. The inference results concern sampling from stratum populations.

\noindent\textbf{Interventions, edits and attribution.} An \emph{intervention} modifies the activation as $H\mapsto H+\Delta(H)$, where $\Delta(H)$ is the \emph{edit}. The intervention keeps the SAE error term. Its edit has mean $\bar u\coloneqq\EE\Delta(H)$, the \emph{mean edit}. A steering intervention at the final position has the fixed edit $\Delta\equiv\varepsilon E_{\rm fin}^\top d_j$. A zero-ablation intervention has the prompt-dependent edit $\Delta_j(H)\coloneqq-\sum_\tau z_{\tau j}E_\tau^\top d_j$, where $z_{\tau j}$ is the activation of latent $j$ at position $\tau$ times the norm of its decoder row. Mean-ablation contributes the edit $(r_{\tau j}-z_{\tau j})E_\tau^\top d_j$ at position $\tau$, for a fixed reference $r_{\tau j}$. With $D(v)\coloneqq F(H+v)-F(H)$, the \emph{finite effect} of the intervention is $\EE D(\Delta(H))$. Its first-order part is the \emph{intervention attribution}, or \emph{attribution},
\begin{equation}\label{eq:effect}
\alpha(\Delta)\coloneqq\EE\bigl[\Delta(H)^\top\nabla F(H)\bigr].
\end{equation}
For steering, we report the unit-direction attribution $d_j^\top E_{\rm fin}\EE\nabla F(H)=\alpha(\Delta)/\varepsilon$ and compare it with the central difference quotient formed from the positive and negative steering interventions. The attribution is the first-order effect of a specified intervention. The following structural assumption, together with the moment conditions stated in each result, makes these objects well defined.

\begin{assumption}[Patching and model-derived target]\label{asm:cut}
The following hold within each stratum.
\begin{enumerate}[label=(\roman*),leftmargin=1.6em,itemsep=0pt]
\item\label{asm:cut:cut} (Patching) The output is $F(H,y^\ast)$ for the fixed downstream computation $F$ from the SAE hook under teacher forcing, with the first position absorbed. An additive patch keeping the SAE error term has outcome $F(H+\Delta,y^\ast)$.
\item\label{asm:cut:target} (Model-derived target) The probe target is $Y=\psi(F(H,y^\ast))+\eta$ with $\psi$ affine, $\EE[\eta\mid H,y^\ast]=0$ and $\EE Y^2<\infty$. Interventions act on $\psi\circ F$, written $F$.
\item\label{asm:cut:reg} (Regularity) $F(\cdot,y^\ast)$ is $C^1$, with $F$ and $\nabla F$ of at most polynomial growth.
\end{enumerate}
\end{assumption}

\Cref{asm:cut}\ref{asm:cut:cut} is the intervention of causal mediation and abstraction \citep{vig2020investigating,geiger2021causal,heimersheim2024how} and \Cref{asm:cut}\ref{asm:cut:target} ties the probe target to the output that the intervention moves. \Cref{asm:cut}\ref{asm:cut:reg} holds for the four models studied here, whose downstream computations from the SAE hook compose normalization layers with a positive stabilizing constant, attention and smooth activations. The SAE encoder, including a discontinuous threshold such as JumpReLU, enters only through the edit and leaves the gradient of $F$ unaffected.

\noindent\textbf{Discovery statistics.} A discovery pipeline applies an SAE with unit decoder directions $d_1,\ldots,d_m$ at one hook to prompts from several tasks and selects candidate latents $S_t$ for each task $t$. We call a quantity \emph{observational} if it is a functional of the law of $(H,Y)$. The discovery statistics considered here are observational. With tasks and fitting choices fixed, they depend only on observed activations and targets. As one concrete pipeline we use the Gaussian-transport implementation in the FDFI package, corresponding to disentangled feature importance with the Bures-Wasserstein map \citep{du2025disentangled,chen2026fdfi}. We refer to it as FDFI and turn its importance tests into a selection by applying the Benjamini-Hochberg procedure within task (\Cref{app:selection}). FDFI instantiates the general problem. The theory covers any observational discovery statistic.

\noindent\textbf{The probe and its implied change.} For a summary $S=AH$ with $A\Sigma_HA^\top\succ0$, the probe coefficient $\beta_S\coloneqq\Cov(S)^{-1}\Cov(S,Y)$ implies the change
\begin{equation}\label{eq:rho}
\rho(\Delta)\coloneqq\bar u^\top A^\top\beta_S
\end{equation}
for an intervention with mean edit $\bar u$, which we call the \emph{probe-implied change}. For a selected latent $j$ the summary is its final-position score, $A=d_j^\top E_{\rm fin}$, with probe coefficient $\beta$. Three quantities thus enter the comparison. The coefficient $\beta$ is observational, $\rho(\Delta)$ is the change it implies for the specified intervention and $\alpha(\Delta)$ is the first-order effect of that intervention.

\noindent\textbf{Data splits.} A \emph{selection split} ($7{,}063$ prompts) fixes the active latents, the tests and the candidates. An \emph{evaluation split} ($1{,}755$ prompts) supplies every gradient, intervention and estimate. Inference treats the samples of different evaluation strata as independent given the selection split (\Cref{asm:split} in \Cref{app:inference}).

Each discovery statistic considered here is a functional of the observational law. The same holds for the probe coefficient. The attribution depends on how $F$ responds to a specified edit. \Cref{sec:identify} asks when the first kind of quantity determines the second.

\section{What an observational discovery identifies}\label{sec:identify}

Natural prompts reveal how $F$ behaves along the activation variation they generate. An intervention can move the activation along a direction that this variation does not contain. Observational data then carry no information about how $F$ responds along that direction. The support of $H$ defines
\begin{equation}\label{eq:V}
V\coloneqq\operatorname{span}\{x-x'\colon x,x'\in\operatorname{supp}(H)\}=\range(\Sigma_H),
\end{equation}
where the second equality holds by \Cref{lem:range}. A law on $n_0$ points has $\dim V\le n_0-1$. Coordinates constant within the stratum, such as those of a shared template prefix, lie in $V^\perp$. We write $x_V$ and $x_{V^\perp}$ for the components of $x$ in $V$ and $V^\perp$.

The \emph{identified set} of a quantity over a class $\mathcal F$ of responses is the set of its values over all responses in $\mathcal F$ that produce the same law of $(H,Y)$ \citep{manski2003partial}. Let $\gamma\coloneqq\EE\nabla F(H)$ denote the mean gradient, which gives every steering intervention the attribution $\bar u^\top\gamma$.

\begin{proposition}[Non-identification of the attribution]\label{prop:discrete}
Let \Cref{asm:cut} hold. Let $\mathcal F$ be a class of responses $G$ that satisfy \Cref{asm:cut}\ref{asm:cut:reg} and $\EE\|\nabla G(H)\|^2<\infty$. Suppose that $\mathcal F$ contains $F$ and is closed under adding affine functions. Let the edit satisfy $\EE\|\Delta(H)\|^2<\infty$. For $\zeta\in V^\perp$ and a fixed $x_1\in\operatorname{supp}(H)$ put $F_\zeta(x)\coloneqq F(x)+\zeta^\top(x-x_1)$. Then
\begin{enumerate}[label=(\alph*),leftmargin=*,itemsep=0pt]
\item\label{prop:discrete:same} $F_\zeta\in\mathcal F$, $F_\zeta=F$ on $\operatorname{supp}(H)$ and every observational quantity takes the same value under $F$ and $F_\zeta$,
\item\label{prop:discrete:attr} the attribution under $F_\zeta$ is $\EE[\Delta(H)^\top\nabla F_\zeta(H)]=\alpha(\Delta)+\bar u^\top\zeta$ and an integrable finite effect shifts by the same amount $\bar u^\top\zeta$,
\item\label{prop:discrete:perp} the identified set of $\gamma$ is invariant under translation by $V^\perp$ and the identified set of $\alpha(\Delta)$ is $\RR$ whenever $\bar u\notin V$.
\end{enumerate}
\end{proposition}

\Cref{prop:discrete} establishes an identification boundary. Identification here is relative to the observational law and the response class. It asks whether all responses in the class that induce that law have the same attribution. For the fixed white-box model, gradients determine the prompt-level attribution contributions and their empirical mean. The same observational law is compatible with every value of the derivative along an edited direction that leaves $V$. The law restricts $\gamma$ at most through its component $\gamma_V$. Consequently, within such a class and unless the mean edit is an exact combination of the activation differences, no observational quantity fixes the size or the sign of the intervention's first-order effect $\alpha(\Delta)$ or of its integrable finite effect. Since $F_\zeta$ keeps the values of $F$ on the support and the Lipschitz constant of its gradient, a bound on that constant leaves this non-identification conclusion unchanged. The construction extends the off-manifold argument of \citet{anders2020fairwashing} for the input gradients of a classifier to the attributions of interventions.

\noindent\textbf{Population and observed support.} \Cref{prop:discrete} holds for every law of $H$. \Cref{thm:gap} decomposes the gap between $\rho$ and $\alpha$ under either law. For the empirical law of a stratum's observed activation-output pairs, \Cref{prop:discrete} leaves the empirical-average attribution unrestricted whenever the empirical mean edit lies outside the span of the observed activation differences. This statement holds within the response class of the proposition. Statistics computed solely from these pairs remain unchanged under its perturbations, with fitting choices and algorithmic randomness held fixed. A small observed span does not imply that the population support is low-dimensional. An evaluation stratum has at most eleven prompts and hence $\dim V\le10$ at its empirical law. Across models, the median share of the mean steering edit's norm outside the observed span is nearly one (\Cref{app:span-check}). The boundary therefore applies to the finite evaluation design actually used.

Observational evidence alone therefore leaves the intervention question open. In a white-box language model, gradients supply the additional information. We next use it to diagnose the gap.

\section{Why the probe and the intervention can disagree}\label{sec:gap}

\Cref{prop:discrete} concerns observational discovery statistics in general. This section studies one canonical translation of observational association into an intervention-scale prediction, the probe-implied change $\rho(\Delta)$. To compare $\alpha(\Delta)$ with the probe, we fix a stratum and a summary $A$ and let
\begin{gather*}
v_A(u)\coloneqq\Sigma_HA^\top(A\Sigma_HA^\top)^{-1}Au,\\
u'\coloneqq A^\top(A\Sigma_HA^\top)^{-1}Au,
\end{gather*}
which satisfy $\Sigma_Hu'=v_A(u)$. We call
\begin{equation}\label{eq:steinres}
\st_w\coloneqq w^\top\bigl\{\Cov(H,Y)-\Sigma_H\gamma\bigr\}
\end{equation}
the \emph{Stein residual} of a direction $w$. The following theorem shows that under \Cref{asm:cut} the gap between the probe-implied change and the intervention's attribution decomposes exactly.

\begin{theorem}[Probe-to-intervention identity]\label{thm:gap}
Assume \Cref{asm:cut}. For every law of $H$ with finite second moments and $\EE\|\nabla F\|^2<\infty$, every summary $A$ with $A\Sigma_HA^\top\succ0$ and every edit with $\EE\|\Delta(H)\|^2<\infty$,
\begin{equation}\label{eq:gap}
\rho(\Delta)-\alpha(\Delta)=\st_{\bar u'}+\cm(\Delta)-\ec(\Delta),
\end{equation}
where $\cm(\Delta)\coloneqq(v_A(\bar u)-\bar u)^\top\gamma$ is the \emph{omitted-variable term} and $\ec(\Delta)\coloneqq\EE[(\Delta(H)-\bar u)^\top(\nabla F(H)-\gamma)]$ is the \emph{edit-gradient covariance}.
\end{theorem}

\Cref{thm:gap} establishes an exact decomposition of the probe's error under its stated moment and nonsingularity hypotheses, including the singular empirical law of a stratum. Each term has a classical counterpart in average derivatives and Stein discrepancies \citep{brillinger1983generalized,stoker1986consistent,goldstein2019nongaussian}, omitted-variable bias \citep{goldberger1991course} and the filter-pattern distinction \citep{haufe2014interpretation}. The identity combines them for a specified intervention. One backward pass per prompt then makes the probe's error estimable term by term within a stratum. Since $v_A(\bar u)\in V$, the part $-\bar u_{V^\perp}^\top\gamma$ of the omitted-variable term is exactly what \Cref{prop:discrete} leaves unidentified. \Cref{tab:terms} reads the three terms.

\begin{table}[t]
\centering\footnotesize
\setlength{\tabcolsep}{3pt}
\renewcommand{\arraystretch}{1.1}
\caption{The three terms of \eqref{eq:gap}, what each compares, when it vanishes and where it arises in SAE pipelines. The empirical stratum laws are finitely supported. The Gaussian case is a benchmark (\Cref{lem:stein}).}\label{tab:terms}
\begin{tabular}{@{}>{\raggedright\arraybackslash}p{0.19\columnwidth}>{\raggedright\arraybackslash}p{0.24\columnwidth}>{\raggedright\arraybackslash}p{0.24\columnwidth}>{\raggedright\arraybackslash}p{0.23\columnwidth}@{}}
\toprule
Term & Compares & Zero when & SAE instance \\
\midrule
Stein residual $\st_{\bar u'}$ & observational regression slope with mean gradient & $F$ globally affine, or $H$ Gaussian with $\Sigma_H\succ0$ & a response curved across prompts \\
Omitted variable $\cm(\Delta)$ & probe's activation pattern with mean edit & $(v_A(\bar u)-\bar u)\perp\gamma$, e.g.\ $v_A(\bar u)=\bar u$ & an edit at positions the probe does not read \\
Edit-gradient covariance $\ec(\Delta)$ & prompt-varying edit with model sensitivity & fixed edit, as for steering & ablation of a rarely firing latent \\
\bottomrule
\end{tabular}
\end{table}

The Stein residual is the error of a regression slope used in place of the mean gradient, since every least-squares slope $\beta_\star$ of $Y$ on $H$ gives $\st_{\bar u'}=v_A(\bar u)^\top(\beta_\star-\gamma)$ (\Cref{lem:range}). For a stratum whose empirical law has affinely independent support and $\EE[Y\mid H]=F(H)$, the plug-in residual is
\begin{equation}\label{eq:secant}
\st_w=(\Sigma_Hw)^\top(b-\gamma_V),
\end{equation}
where $b\in V$ is the slope of the affine function that interpolates $F$ on the support (\Cref{lem:finite-stein}). It then compares a secant of the response with its mean gradient. When $\Sigma_H\succ0$, a linear probe is an average-derivative estimate under the Gaussian score with the activation's moments. Its Stein residual vanishes for Gaussian activations (\Cref{lem:stein,app:terms}).

The omitted-variable term is the mean gradient along the difference between the probe's activation pattern $v_A(\bar u)$ \citep{haufe2014interpretation} and the mean edit that the intervention applies. For an edit along the read direction it has the form of omitted-variable bias in a short regression \citep{goldberger1991course}, with the mean gradient in place of the long-regression coefficient. Feature splitting and absorption \citep{bricken2023monosemanticity,chanin2024absorption} can spread a concept across latents with dependent scores (\Cref{lem:vif}). For example, let $H=(X,X)$ with $\VV(X)>0$, $F(x)=x_2$, a probe on the first coordinate and the steering edit $e_1$. Then $\rho=1$ while $\alpha=0$. The whole gap is the omitted-variable term, since the probe's pattern moves both coordinates and the edit moves only the first. \Cref{app:terms} gives closed-form examples of all three terms. The edit-gradient covariance separates the mean of $\Delta(H)^\top\nabla F(H)$ from the product of the means. It vanishes for steering, whose edit is fixed. It can be nonzero for ablations, whose edit varies with the prompt.

\noindent\textbf{Ablations.} Pipelines also use ablation interventions, which remove a latent's own activation. The edit they induce varies with the prompt, including at positions that a final-position probe does not read. Let $\Gamma_\tau\coloneqq d_j^\top E_\tau\nabla F$ be the directional derivative at position $\tau$.

\begin{corollary}[Ablations]\label{cor:ablate}
Assume the hypotheses of \Cref{thm:gap} and let $\Delta_j$ be the edit of the zero-ablation intervention on latent $j$.
\begin{enumerate}[label=(\alph*),leftmargin=*,itemsep=0pt]
\item\label{cor:ablate:dec} The attribution splits as
\begin{equation}\label{eq:ablate}
\alpha(\Delta_j)=L_j\coloneqq M_j+C_j,
\end{equation}
where $M_j\coloneqq-\sum_\tau\EE[z_{\tau j}]\EE[\Gamma_\tau]$ and $C_j\coloneqq-\sum_\tau\Cov(z_{\tau j},\Gamma_\tau)$.
\item\label{cor:ablate:blind} If $A=d_j^\top E_{\rm fin}$ and $z_{{\rm fin},j}=0$ almost surely, then $\rho(\Delta_j)=0$ and $\rho(\Delta_j)-\alpha(\Delta_j)=-\alpha(\Delta_j)$.
\end{enumerate}
\end{corollary}

\Cref{cor:ablate} establishes that the attribution of an ablation intervention adds to the mean-activation term $M_j$ a covariance $C_j$ of unrestricted sign. Rare firing makes the code's coefficient of variation large and can make $|C_j|$ exceed $|M_j|$ (\Cref{app:proofs:ablate}). Part~\ref{cor:ablate:blind} shows that the probe misses the whole attribution when the ablation edits only positions the probe does not read. The general statement for any set of edited positions and for mean-ablation is in \Cref{app:proofs:ablate}.

The identity is exact. \Cref{sec:results} tests its consequences in practice. Are its terms large? Does the attribution predict realized interventions? Does the distinction change which latents one would choose?

\section{Experiments}\label{sec:results}

\noindent\textbf{Setup.} We study four public model-SAE pairs, GPT-2 \citep{radford2019language}, Gemma-2-2B and Gemma-2-9B \citep{gemmateam2024gemma2} and Llama-3.1-8B \citep{dubey2024llama3}, each with a residual-stream SAE at one layer \citep{bloom2024gpt2residualsaes,lieberum2024llamascope,he2024llamascope}. Eleven tasks cover factual recall (country capitals in free and multiple-choice form, country languages, birthplaces and occupations), lexical tasks (antonyms and last letters), in-context tasks (few-shot induction and indirect object identification) and arithmetic (simple addition and two-digit subtraction) (\Cref{app:protocols}). The antonym task, with no evaluation stratum of at least three prompts, enters only the two-prompt ablation sensitivity analysis and the held-out set-ablation comparison. Unless stated otherwise, quantities are estimated on the held-out evaluation split within strata of at least three prompts, in within-task standard deviations of the log-probability (\Cref{app:protocols}). Stratum-level $t$ intervals use a many-strata normal approximation (\Cref{thm:inference}) and in simulation at the observed stratum sizes cover near the nominal level for the ablation terms and below it for the Stein residual (\Cref{app:coverage}). Other intervals are percentile bootstrap intervals. The fixed-direction analysis, the coefficient-selected steering and the single-latent ablations are prespecified. The maximum-activation steering and the set-ablation comparison of selection rules are post hoc, specified after the aggregate results and before evaluation. The probe-to-ablation breakdown, the averaged-summary analysis, the ranking comparison and the worked example are post hoc (\Cref{app:prespec}). Because steering uses step sizes scaled on the evaluation split, its comparisons are descriptive. The results answer three questions. How large is the gap? Does the attribution predict realized interventions? Does the distinction change which latents one would choose?

\begin{table*}[t]
\centering
\footnotesize
\setlength{\tabcolsep}{2.4pt}
\renewcommand{\arraystretch}{1.15}
\caption{Main results on held-out prompts. Columns Pairs through $|\cm/\beta|$ use FDFI-selected directions under unit steering. Ratios are medians over pairs, with intervals that resample strata within task. Column $<0.1$ is the share with $|\alpha/\beta|<0.1$ and $\rho\approx0$ is the post hoc share of single-latent zero-ablations with $|\rho|\le10^{-3}<|\alpha|$. Steering is descriptive, on coefficient-selected directions ($90$, $100$, $100$ and $90$ pairs), with $T$ the central difference quotient of the realized steering interventions and $|T/\beta|$ in units of $10^{-3}$. Ablation gives the reduction in mean absolute error from the probe to the gradient prediction. Sign gives the gradient/probe sign agreement among rows with nonzero realized effects, averaged equally over tasks, with a zero prediction counted as a disagreement. Choosing latents is descriptive. Top-1 counts the units, out of the number in parentheses, in which the FDFI statistic, $|\beta|$ or $|\alpha|$ ranks first the latent with the largest $|T|$ at one maximum activation. The last columns give model means of the attribution-set minus comparator contrast in absolute held-out mean set-ablation effects, in standardized units, with all five comparators in \Cref{app:choice}. $^\ast$Fewer than thirty pairs.}\label{tab:main}
\resizebox{\textwidth}{!}{%
\begin{tabular}{@{}l rrrrrr rrrr rrr@{}}
\toprule
& \multicolumn{6}{c}{Size of the gap} & \multicolumn{4}{c}{Realized interventions} & \multicolumn{3}{c}{Choosing latents} \\
\cmidrule(lr){2-7}\cmidrule(lr){8-11}\cmidrule(l){12-14}
& & & & & & & \multicolumn{2}{c}{Steering} & \multicolumn{2}{c}{Ablation} & & \multicolumn{2}{c}{Sets, attr.\ $-$} \\
\cmidrule(lr){8-9}\cmidrule(lr){10-11}\cmidrule(l){13-14}
Model & Pairs & $|\alpha/\beta|$ (\%) [CI] & $<0.1$ & $|\st/\beta|$ & $|\cm/\beta|$ & $\rho\approx0$ & $|T/\alpha|$ & $|T/\beta|$ & MAE red. & Sign (\%) & Top-1 & $|\beta|$ & FDFI \\
\midrule
GPT-2 & 46 & 4.6 [3.0, 5.9] & 0.739 & 0.614 & 0.514 & 0.370 & 1.000 & 7.60 & 0.303 & 92/21 & 1/0/4 (6) & 0.454 & 0.290 \\
Gemma-2-2B & 230 & 3.2 [2.2, 3.6] & 0.770 & 0.922 & 0.593 & 0.296 & 1.000 & 3.68 & 0.119 & 91/22 & 0/0/4 (9) & 0.506 & $-0.209$ \\
Gemma-2-9B & 457 & 2.8 [1.9, 3.5] & 0.751 & 0.728 & 0.463 & 0.287 & 1.001 & 2.50 & 0.114 & 93/17 & 0/0/4 (9) & 2.137 & 1.380 \\
Llama-3.1-8B & 23$^\ast$ & 3.2 [2.2, 5.1] & 0.826 & 0.583 & 0.406 & 0.043 & 1.000 & 3.36 & 0.300 & 97/35 & 2/1/4 (4) & 0.309 & 0.206 \\
\bottomrule
\end{tabular}}
\end{table*}

\begin{figure*}[t]
\centering
\includegraphics[width=\textwidth]{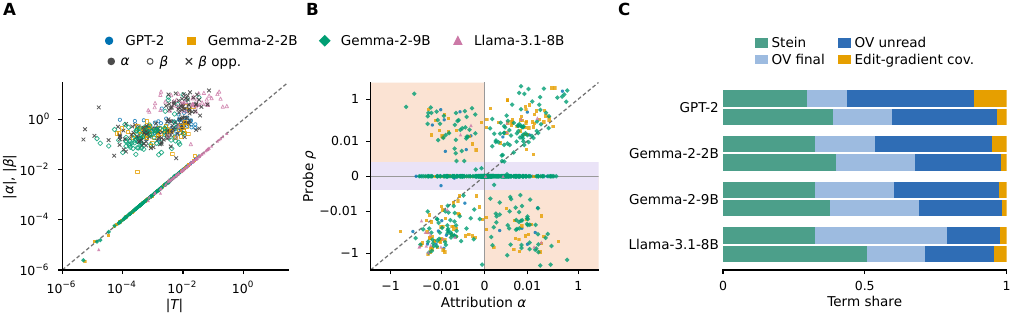}
\caption{Probe-implied changes and intervention attributions on held-out prompts, in standardized units. (A) Magnitudes of the unit-direction attribution $\alpha$ (filled) and the probe coefficient $\beta$ (open, crosses for the sign opposite to $T$) against the central difference quotient $|T|$ of steering at $\pm0.25$ within-stratum standard deviations for the $380$ coefficient-selected pairs. (B) Probe-implied change $\rho$ against attribution $\alpha$ on symmetric-log axes under zero-ablation of the $756$ FDFI-selected pairs. Shading marks opposite signs and negligible predictions ($|\rho|<10^{-3}$). Dashed diagonals mark equality. (C) Shares of the summed absolute terms of \eqref{eq:gap} for the final-position probe, with the omitted-variable (OV) term split by position and zero- and mean-ablation as upper and lower bars for each model. Panel A is descriptive. Panels B and C are post hoc and descriptive (\Cref{app:x6,app:x8}).}\label{fig:headline}
\end{figure*}

\noindent\textbf{Size of the gap.} For FDFI-selected directions under unit steering, the median ratio of attribution to probe-coefficient magnitude is $2.8$ to $4.6\%$ across models, with the ratio below 10\% for most pairs (\Cref{tab:main}). This mismatch is not specific to FDFI. When we instead select 10 active latents per task by selection-split coefficient magnitude, their coefficients exceed realized small-step steering quotients by more than two orders of magnitude in the model-specific medians. The decomposition points to why: under both selections, the Stein residual and omitted-variable term are each comparable in magnitude to the coefficient. The discrepancy also persists beyond the final-position probe. In a post hoc descriptive analysis, probes on position-averaged decoder scores and FDFI's averaged-code input remove most of the blindness to unread positions yet still overstate the attribution in aggregate (\Cref{app:fdfi-summary}). Finally, sampling-noise analyses and prespecified robustness variants support the same qualitative conclusion (\Cref{app:noise,app:e1r}).

\noindent\textbf{Where the gap arises.} Ablation interventions reveal how the terms in \Cref{cor:ablate} arise in SAE pipelines. On GPT-2 and the Gemma models, about a third of FDFI-selected pairs have a non-negligible attribution but a zero or negligible probe-implied change (\Cref{fig:headline}B, \Cref{tab:main}). Almost all of these latents never fire at the final position, which is the only position read by the probe. For these pairs, the gap arises at unread positions and is dominated by the omitted-variable term, with the edit-gradient covariance accounting for much of the remainder (\Cref{fig:headline}C, \Cref{app:position}). More broadly, under zero-ablation, the covariance term \(C_j\) exceeds the mean-activation term in magnitude for a substantial share of selected pairs (\Cref{app:e8}).

\noindent\textbf{Predicting realized interventions.} We first test whether attribution predicts realized interventions in the local regime. For coefficient-selected directions, we steer held-out prompts by a quarter of a within-stratum standard deviation of the latent score. The unit-direction attribution closely matches the realized central difference quotient \(T\) and is closer than the probe coefficient for every pair (\Cref{fig:headline}A, \Cref{tab:main}, \Cref{app:x6}). We next examine how far this first-order approximation remains accurate by increasing the steering displacement for FDFI-selected directions. At a displacement scaled to each latent's maximum observed activation, attribution agrees in sign with \(T\) for nearly all pairs and remains closer than the coefficient for almost all. At five times this scale, the approximation breaks down for most pairs, with \(T\) differing from the attribution by more than half its magnitude. Single-latent ablations provide an independent check. Prompt-specific gradient predictions have lower mean absolute error than probe-based predictions on every model and substantially higher sign agreement among ablations with nonzero realized effects (\Cref{tab:main}, \Cref{app:singletons}).

\noindent\textbf{Choosing latents for intervention.} The gap affects feature choice as well as effect prediction. Within FDFI-selected candidates, attribution more often ranks first the latent with the largest realized steering response than either the FDFI statistic or the probe coefficient (\Cref{app:choice-rank}). This agreement is nearly universal at small displacements and remains strong at a displacement scaled to the latent's maximum observed activation, before weakening as the first-order approximation breaks down at larger scales. We then ask whether this advantage carries over to held-out feature selection. Using training folds to select up to 10 latents per task, attribution-selected sets produce larger absolute mean ablation effects on held-out prompts than correlation-, rank-test- or coefficient-selected sets in every model-task unit (\Cref{app:choice-sets}). These contrasts remain positive after normalization by edit norm and their two-stage intervals exclude zero. Attribution-selected sets also exceed class-mean-difference sets in most units. Compared with FDFI selection, the contrast is positive in most units and in three of four models.

\section{Discussion}\label{sec:discussion}

Observational discovery and intervention attribution answer different questions. Empirical studies report that activation-side evidence about latents need not predict their effects under intervention \citep{arad2025saes, tiwari2026decodability}. We formalize this separation with an intervention-specific estimand, an identification boundary within the response class of \Cref{prop:discrete} and an exact decomposition (\Cref{thm:gap}). In the evaluated pipelines the gap changes which latents one would choose.

Discovery statistics remain useful for finding candidate features, whereas claims about the consequence of a specified intervention rest on evidence about that intervention. In a white-box model one backward pass per prompt computes the empirical attribution. For each such intervention one can report its attribution with a held-out interval, the terms of \eqref{eq:gap} and the share of the mean edit outside the observed span.

\noindent\textbf{Scope and extensions.} The attribution is the first-order effect of a specified intervention. Whether a latent is the sole mechanism behind a behaviour or necessary or sufficient for it is a broader mechanistic question. The identity holds for any layer, dictionary, linear readout and intervention that keeps the SAE error term. Our experiments cover four model-SAE pairs at one layer each and eleven tasks, with few selected pairs for Llama-3.1-8B, teacher-forced outputs and steering up to five maximum activations. Two directions follow. For inference, the stratum-level construction targets a fixed-weight population attribution under independent sampling across strata (\Cref{app:inference}). Inference over template clusters would relax that independence. Anytime-valid confidence sequences under adaptive intervention sampling \citep{asiaee2026certified} suggest a route to inference for adaptively chosen interventions. For finite effects, the intervention must also be evaluated at the scale of interest. Downstream nonlinearities can make first-order approximations inaccurate \citep{wang2026attribution}, motivating identification and inference beyond local attribution.

\bibliography{references}

\clearpage
\appendix
\begin{center}{\Large\bfseries Appendix}\end{center}

\counterwithin{lemma}{section}
\counterwithin{proposition}{section}
\counterwithin{remark}{section}
\counterwithin{figure}{section}
\counterwithin{table}{section}
\counterwithin{theorem}{section}
\counterwithin{corollary}{section}

\input{sections/app_notation}
\FloatBarrier
\input{sections/app_proofs}
\FloatBarrier
\input{sections/app_inference}
\FloatBarrier
\input{sections/app_empirics}
\FloatBarrier
\input{sections/app_reports}
\FloatBarrier
\input{sections/app_finite}
\FloatBarrier
\input{sections/app_choice}
\FloatBarrier
\input{sections/app_case}

\typeout{get arXiv to do 4 passes: Label(s) may have changed. Rerun}
\end{document}

%% file: sections/app_notation.tex
\section{Notation and organization}\label{app:notation}

\paragraph{Organization.} \Cref{app:proofs} proves the results of the main text, reads the terms of the probe-to-intervention identity and extends \Cref{cor:ablate} to any set of edited positions. \Cref{app:inference} states the split design, proves the many-strata inference result behind the stratum-level intervals and measures their coverage by simulation. \Cref{app:empirics} describes the tasks, models, protocols and additional analyses, the FDFI selection step and the observed span of each stratum. \Cref{app:reports,app:finite} compare the probe with the attribution under steering and ablation, test the gap against sampling noise and compare both with realized ablations. \Cref{app:case} details the worked example.

\paragraph{Notation.} The index $i$ runs over prompts, $j$ over latents, $k$ over strata, $t$ over tasks and $\tau$ over positions. In the FDFI statistics of \Cref{app:selection}, $l$ runs over active latents and $k$ over disentangled directions. A subscript $i$ evaluates a quantity at the activation $H_i$ of prompt $i$ and a subscript $j$ refers to the single-latent edit $\Delta_j$ of latent $j$. Without a prompt index a quantity is defined at a fixed activation $H$ or is an expectation over the stratum law, as $L_j$ is. In the tables $r$ denotes a Pearson correlation. The table below lists the symbols used in more than one appendix.

\medskip
{\footnotesize
\renewcommand{\arraystretch}{1.05}
\setlength{\tabcolsep}{3pt}
\noindent\begin{tabular}{@{}>{\raggedright\arraybackslash}p{0.10\textwidth}>{\raggedright\arraybackslash}p{0.38\textwidth}>{\raggedright\arraybackslash}p{0.10\textwidth}>{\raggedright\arraybackslash}p{0.38\textwidth}@{}}
\toprule
Symbol & Meaning & Symbol & Meaning \\
\midrule
$H$, $h$, $E_\tau$ & activations at all positions after the first, final block $E_{\rm fin}H$, block selector &
$\rho(\Delta)$ & probe-implied change $\bar u^\top A^\top\beta_S$ \\
$F$, $Y$ & standardized log-probability, target $F(H)+\eta$ &
$u'$ & $A^\top(A\Sigma_HA^\top)^{-1}Au$ \\
$\mu$, $\Sigma_H$, $\gamma$ & mean and covariance of $H$, mean gradient &
$\st_w$, $\cm$, $\ec$ & the three terms of \eqref{eq:gap} \\
$V$, $x_V$, $x_{V^\perp}$ & $\range(\Sigma_H)$, components in $V$ and $V^\perp$ &
$\mathrm{VIF}_j$ & variance inflation of a score (\Cref{lem:vif}) \\
$d_j$, $\beta$ & unit decoder direction, coefficient of $d_j^\top h$ &
$M_j$, $C_j$, $L_j$ & the terms of \eqref{eq:ablate}, $L_j=M_j+C_j$ \\
$\Delta(H)$, $\bar u$ & edit of an intervention keeping the SAE error term, mean edit &
$T$, $\sigma_{tj}$ & central difference quotient of steering at step $0.25\,\sigma_{tj}$, pooled within-stratum standard deviation of the score \\
$D(v)$, $\alpha(\Delta)$ & $F(H+v)-F(H)$, attribution \eqref{eq:effect} &
$K_s$, $n_k$, $\omega_k$ & number of strata, stratum size, stratum weight \\
$S=AH$, $\beta_S$ & summary, coefficient $\Cov(S)^{-1}\Cov(S,Y)$ &
$q$ & Benjamini-Hochberg or Benjamini-Yekutieli level \\
\bottomrule
\end{tabular}
}

%% file: sections/app_proofs.tex
\section{Proofs and supporting results}\label{app:proofs}

Every statement below holds for any law of $H$ that meets its hypotheses, the population law of a stratum and the empirical law of its observed prompts alike. We write $\mu\coloneqq\EE H$ and $Y=F(H)+\eta$ with $\EE[\eta\mid H]=0$, where within a stratum $F$ stands for $\psi\circ F$ and absorbs the fixed first position (\Cref{asm:cut}).

\subsection{Supporting lemmas}\label{app:lemmas}

For a positive $C^1$ density $\varrho$ on $\RR^q$ we write $s_\varrho\coloneqq-\nabla\log\varrho$ for its score.

\begin{lemma}[Integration by parts, \citealp{stein1981estimation,stoker1986consistent}]\label{lem:stein}
Let $\varrho$ be a positive $C^1$ density on $\RR^q$ with score $s_\varrho$.
\begin{enumerate}[label=(\alph*),leftmargin=*,itemsep=0pt]
\item\label{lem:stein:ibp} If $f\in W^{1,1}_{\rm loc}$ and $\varrho\nabla f$, $f\varrho\,s_\varrho$ and $f\varrho$ lie in $L^1$, then $\EE_\varrho[\nabla f]=\EE_\varrho[f\,s_\varrho]$.
\item\label{lem:stein:gauss} Let $\varrho=\cN(\mu,\Sigma)$ with $\Sigma\succ0$ and $X\sim\varrho$. The integrability conditions of \ref{lem:stein:ibp} then reduce to $\EE\|\nabla f(X)\|<\infty$ and $\EE[|f(X)|\,\|X-\mu\|]<\infty$. Under them, $\Cov(X,f(X))=\Sigma\,\EE\nabla f(X)$.
\item\label{lem:stein:act} Under the hypotheses of \Cref{thm:gap}, a Gaussian activation law with $\Sigma_H\succ0$ has $\Cov(H,Y)=\Sigma_H\gamma$. Every Stein residual $\st_w$ then vanishes and $\gamma=\Sigma_H^{-1}\Cov(H,Y)$.
\end{enumerate}
\end{lemma}

\begin{proof}[Proof of \Cref{lem:stein}]
Parts \ref{lem:stein:ibp} and \ref{lem:stein:gauss} are the classical identity of \citet{stein1981estimation} and \citet{stoker1986consistent}. For part \ref{lem:stein:act} we apply \ref{lem:stein:gauss} with $f=F$ and $\Sigma=\Sigma_H$. The response is $C^1$ and hence lies in $W^{1,1}_{\rm loc}$. The hypotheses of \Cref{thm:gap} give $\EE\|\nabla F(H)\|^2<\infty$. Since $F(H)=\EE[Y\mid H]$, Jensen's inequality gives $\EE F(H)^2\le\EE Y^2<\infty$ and the Cauchy-Schwarz inequality then gives $\EE[|F(H)|\,\|H-\mu\|]<\infty$. Part \ref{lem:stein:gauss} therefore yields $\Cov(H,F(H))=\Sigma_H\gamma$. The noise $\eta=Y-F(H)$ is square integrable with $\EE[\eta\mid H]=0$, which gives $\Cov(H,\eta)=\EE[(H-\mu)\EE(\eta\mid H)]=0$. Hence $\Cov(H,Y)=\Cov(H,F(H))+\Cov(H,\eta)=\Sigma_H\gamma$. By \eqref{eq:steinres}, $\st_w=w^\top\{\Cov(H,Y)-\Sigma_H\gamma\}=0$ for every $w$. Inverting $\Sigma_H$ gives $\gamma=\Sigma_H^{-1}\Cov(H,Y)$, the Gaussian identification of \citet{brillinger1983generalized} and \citet{stoker1986consistent}. This completes the proof.
\end{proof}

\begin{lemma}[Support, covariance and least squares]\label{lem:range}
Let $H$ have finite second moments, mean $\mu$ and covariance $\Sigma_H$. Let $V$ be the span of the differences of support points, as in \eqref{eq:V}.
\begin{enumerate}[label=(\alph*),leftmargin=*,itemsep=0pt]
\item\label{lem:range:range} $\range(\Sigma_H)=V$ and $H-\mu\in V$ almost surely.
\item\label{lem:range:ne} If $\EE Y^2<\infty$, then $\Cov(H,Y)\in V$ and the normal equations $\Sigma_H\beta=\Cov(H,Y)$ have a solution.
\item\label{lem:range:sol} If $\EE Y^2<\infty$, every solution $\beta_\star$ of the normal equations gives $\rho(\Delta)=v_A(\bar u)^\top\beta_\star$ for every summary $A$ with $A\Sigma_HA^\top\succ0$. If in addition $\EE\|\nabla F(H)\|<\infty$, it also gives $\st_w=(\Sigma_Hw)^\top(\beta_\star-\gamma)$ for every $w$.
\item\label{lem:range:dim} A law on $n_0$ points has $\dim V\le n_0-1$. A law with a density has $V=\RR^p$.
\end{enumerate}
\end{lemma}

\begin{proof}[Proof of \Cref{lem:range}]
The proof proceeds in three parts. First, we identify the kernel of $\Sigma_H$ with $V^\perp$, which gives \ref{lem:range:range}. Second, we solve the normal equations and rewrite the Stein residual and the probe-implied change, which gives \ref{lem:range:ne} and \ref{lem:range:sol}. Third, we bound the dimension of $V$, which gives \ref{lem:range:dim}. Throughout, $\mathcal X\coloneqq\operatorname{supp}(H)$ is the smallest closed set of full measure and $x_1\in\mathcal X$ is fixed. Since $x-x'=(x-x_1)-(x'-x_1)$, the vectors $x-x_1$ with $x\in\mathcal X$ span $V$.

\textbf{Part 1 (Kernel and range of the covariance).}
Let $w\perp V$. Then $w^\top x=w^\top x_1$ for every $x\in\mathcal X$. Since $H\in\mathcal X$ almost surely, $w^\top H=w^\top x_1$ almost surely. Taking expectations gives $w^\top x_1=w^\top\mu$ and hence $w^\top(H-\mu)=0$ almost surely. Because $\Sigma_H$ is positive semidefinite, $\|\Sigma_H^{1/2}w\|^2=w^\top\Sigma_Hw=\VV(w^\top H)=0$ forces $\Sigma_Hw=0$. Conversely, let $\Sigma_Hw=0$. Then $\VV(w^\top H)=0$ and $w^\top H=w^\top\mu$ almost surely. The closed hyperplane $\{x\colon w^\top x=w^\top\mu\}$ has full measure and therefore contains $\mathcal X$. Hence $w^\top(x-x')=0$ for all $x,x'\in\mathcal X$, which gives $w\perp V$. Thus $\ker\Sigma_H=V^\perp$ and the symmetry of $\Sigma_H$ gives $\range(\Sigma_H)=(\ker\Sigma_H)^\perp=V$. Applying the first step to each vector of a basis of $V^\perp$ gives $H-\mu\in V$ almost surely.

\textbf{Part 2 (Normal equations, Stein residual and probe-implied change).}
By the Cauchy-Schwarz inequality, $\EE[\|H-\mu\|\,|Y|]<\infty$. Since $\EE(H-\mu)=0$, $\Cov(H,Y)=\EE[(H-\mu)Y]$, which lies in $V=\range(\Sigma_H)$ because $H-\mu\in V$ almost surely. The normal equations are therefore solvable. Let $\beta_\star$ be a solution. If $\EE\|\nabla F(H)\|<\infty$, the mean gradient $\gamma$ exists. Substituting the normal equations into \eqref{eq:steinres} then gives $\st_w=w^\top\Sigma_H\beta_\star-w^\top\Sigma_H\gamma=(\Sigma_Hw)^\top(\beta_\star-\gamma)$ by the symmetry of $\Sigma_H$. For the probe-implied change, $\beta_S=(A\Sigma_HA^\top)^{-1}A\Cov(H,Y)$ and the definition of $\bar u'$ give
\begin{align*}
\rho(\Delta)
&=\bar u^\top A^\top\beta_S\\
&=\bar u'^\top\Cov(H,Y)\\
&=\bar u'^\top\Sigma_H\beta_\star\\
&=v_A(\bar u)^\top\beta_\star .
\end{align*}
The last step follows from $\Sigma_H\bar u'=v_A(\bar u)$ and the symmetry of $\Sigma_H$.

\textbf{Part 3 (Dimension of the support span).}
Let $\mathcal X=\{x_1,\ldots,x_{n_0}\}$. The $n_0-1$ vectors $x_i-x_1$ then span $V$, which gives $\dim V\le n_0-1$. Now let the law of $H$ have a density. Since $H\in\mathcal X\subset x_1+V$ almost surely, the affine subspace $x_1+V$ has probability one. A proper affine subspace is Lebesgue-null and therefore carries no mass under a law with a density. Hence $V=\RR^p$. This completes the proof.
\end{proof}

\subsection{Proof of \texorpdfstring{\Cref{prop:discrete}}{Proposition 1}}\label{app:nonid}

\begin{proof}[Proof of \Cref{prop:discrete}]
The proof proceeds in three parts. First, we show that $F_\zeta$ belongs to $\mathcal F$ and produces the law of $(H,Y)$, which gives \ref{prop:discrete:same}. Second, we compute the gradient of $F_\zeta$, its attribution and its finite effect, which gives \ref{prop:discrete:attr}. Third, we apply the construction to every observationally equivalent response, which gives \ref{prop:discrete:perp}. Since $\EE\|\Delta(H)\|^2<\infty$ and $\EE\|\nabla G(H)\|^2<\infty$ for every $G\in\mathcal F$, the Cauchy-Schwarz inequality makes every attribution below finite.

\textbf{Part 1 (Observational equivalence).}
Since $x\mapsto\zeta^\top(x-x_1)$ is affine, closure of $\mathcal F$ under adding affine functions places $F_\zeta$ in $\mathcal F$. For $x\in\operatorname{supp}(H)$ both $x$ and $x_1$ lie in the support, which gives $x-x_1\in V$. Because $\zeta\in V^\perp$, $\zeta^\top(x-x_1)=0$ and hence $F_\zeta(x)=F(x)$. Since $H\in\operatorname{supp}(H)$ almost surely, $F_\zeta(H)=F(H)$ almost surely. With the same noise $\eta$, for which $\EE[\eta\mid H]=0$ still holds, the target $F_\zeta(H)+\eta$ equals $Y$ almost surely. The response $F_\zeta$ thus produces the law of $(H,Y)$. Every observational quantity is a functional of this law and therefore takes the same value under $F$ and $F_\zeta$.

\textbf{Part 2 (Gradient, attribution and finite effect).}
Since $F$ is $C^1$, $\nabla F_\zeta=\nabla F+\zeta$ everywhere with a constant difference $\zeta$. Hence $\nabla F_\zeta$ has the Lipschitz constant of $\nabla F$ and $\nabla^2F_\zeta=\nabla^2F$ wherever $F$ is twice differentiable. Taking expectations gives $\EE\nabla F_\zeta(H)=\gamma+\zeta$. The attribution under $F_\zeta$ is therefore
\begin{align*}
\EE\bigl[\Delta(H)^\top\nabla F_\zeta(H)\bigr]
&=\EE\bigl[\Delta(H)^\top\nabla F(H)\bigr]+\EE[\Delta(H)]^\top\zeta\\
&=\alpha(\Delta)+\bar u^\top\zeta .
\end{align*}
The last step follows from the definitions of $\alpha(\Delta)$ and $\bar u$. For the finite effect write $D_\zeta(v)\coloneqq F_\zeta(H+v)-F_\zeta(H)$. Since $\zeta^\top(H+v-x_1)-\zeta^\top(H-x_1)=\zeta^\top v$, the definition of $F_\zeta$ gives $D_\zeta(\Delta(H))=D(\Delta(H))+\zeta^\top\Delta(H)$. The increment $\zeta^\top\Delta(H)$ is integrable because $\EE\|\Delta(H)\|^2<\infty$. Hence $D_\zeta(\Delta(H))$ is integrable exactly when $D(\Delta(H))$ is. Taking expectations gives $\EE D_\zeta(\Delta(H))=\EE D(\Delta(H))+\bar u^\top\zeta$.

\textbf{Part 3 (Identified sets).}
Write $\mathcal G$ for the identified set of $\gamma$. Let $\tilde F\in\mathcal F$ produce the law of $(H,Y)$. Parts 1 and 2 with $F$ replaced by $\tilde F$ show that $\tilde F_\zeta\in\mathcal F$ produces the same law and has mean gradient $\EE\nabla\tilde F(H)+\zeta$. Hence $\mathcal G+\zeta\subset\mathcal G$ for every $\zeta\in V^\perp$. Applying this inclusion to $-\zeta\in V^\perp$ gives $\mathcal G\subset\mathcal G+\zeta$. Consequently $\mathcal G+\zeta=\mathcal G$, which is the invariance under translation by $V^\perp$. Now let $\bar u\notin V$. Then $\bar u_{V^\perp}\ne0$ and $\zeta=t\,\bar u_{V^\perp}$ lies in $V^\perp$ for every $t\in\RR$. By Part 2 this choice gives the attribution $\alpha(\Delta)+t\,\bar u^\top\bar u_{V^\perp}=\alpha(\Delta)+t\,\|\bar u_{V^\perp}\|^2$, which takes every real value as $t$ ranges over $\RR$. The identified set of $\alpha(\Delta)$ is therefore $\RR$. This completes the proof.
\end{proof}

\noindent\textbf{A two-prompt example.} Let a stratum contain two prompts with activations $x_1\ne x_2$. Fix a unit vector $v\perp(x_1-x_2)$. Such a vector exists whenever $p\ge2$. Let $\varepsilon$ be a nonzero steering scale. Within the response class of \Cref{prop:discrete}, the responses $F_t(x)=F(x)+t\,v^\top(x-x_1)$, $t\in\RR$, agree on both prompts and preserve every observational quantity. For the fixed edit $\Delta=\varepsilon v$, their attributions are $\varepsilon v^\top\gamma+\varepsilon t$. Thus the same observations permit every real attribution for this specified intervention.

\subsection{Proof of \texorpdfstring{\Cref{thm:gap}}{Theorem 1}}\label{app:proofs:gap}

\begin{proof}[Proof of \Cref{thm:gap}]
The proof proceeds in three parts. First, we expand the attribution about the means of $\Delta(H)$ and $\nabla F(H)$. Second, we rewrite the probe-implied change through the normal equations of the probe. Third, we subtract the two expansions. By the Cauchy-Schwarz inequality and the second moments, every expectation below is finite.

\textbf{Part 1 (Expansion of the attribution).}
Writing $\Delta(H)=\bar u+(\Delta(H)-\bar u)$ and $\nabla F(H)=\gamma+(\nabla F(H)-\gamma)$ and expanding the product yields
\begin{align*}
\alpha(\Delta)
&=\bar u^\top\gamma+\EE\bigl[(\Delta(H)-\bar u)^\top\gamma\bigr]+\EE\bigl[\bar u^\top(\nabla F(H)-\gamma)\bigr]+\EE\bigl[(\Delta(H)-\bar u)^\top(\nabla F(H)-\gamma)\bigr]\\
&=\bar u^\top\gamma+\ec(\Delta).
\end{align*}
The last step follows from $\EE[\Delta(H)-\bar u]=0$, $\EE[\nabla F(H)-\gamma]=0$ and the definition of $\ec(\Delta)$.

\textbf{Part 2 (Expansion of the probe-implied change).}
Since $A\Sigma_HA^\top\succ0$, the normal equations $A\Sigma_HA^\top\beta_S=A\Cov(H,Y)$ of the probe give $\beta_S=(A\Sigma_HA^\top)^{-1}A\Cov(H,Y)$. Substituting $\beta_S$ into \eqref{eq:rho} yields
\begin{align*}
\rho(\Delta)
&=\bar u^\top A^\top(A\Sigma_HA^\top)^{-1}A\Cov(H,Y)\\
&=\bar u'^\top\Cov(H,Y)\\
&=\bar u'^\top\bigl\{\Cov(H,Y)-\Sigma_H\gamma\bigr\}+\bar u'^\top\Sigma_H\gamma\\
&=\st_{\bar u'}+v_A(\bar u)^\top\gamma .
\end{align*}
The second step uses the definition of $\bar u'$ and the symmetry of $(A\Sigma_HA^\top)^{-1}$. The last step follows from \eqref{eq:steinres}, $\Sigma_H\bar u'=v_A(\bar u)$ and the symmetry of $\Sigma_H$.

\textbf{Part 3 (The identity).}
Subtracting the expansion of Part 1 from that of Part 2 yields $\rho(\Delta)-\alpha(\Delta)=\st_{\bar u'}+(v_A(\bar u)-\bar u)^\top\gamma-\ec(\Delta)$, which is \eqref{eq:gap} by the definition of $\cm(\Delta)$. This completes the proof.
\end{proof}

The derivation uses only the means and covariances of $H$, $Y$, $\Delta(H)$ and $\nabla F(H)$ and requires of $\Sigma_H$ no more than $A\Sigma_HA^\top\succ0$. Since the empirical law of a stratum has finite moments of every order, the identity holds for it whenever $A\Sigma_HA^\top\succ0$. Averaging the stratum-specific identities with fixed stratum weights $\omega_k$ gives \eqref{eq:gap} for averages over strata, with each term the $\omega$-weighted average of its stratum-specific values.

At the empirical law of a stratum of $n_k$ prompts, $A\Sigma_HA^\top$ has rank at most $\dim V\le n_k-1$ (\Cref{lem:range}). A summary with more rows, such as the scores of one latent at several positions, therefore has a singular $A\Sigma_HA^\top$ in small strata. For such summaries we pool within-stratum moments over the strata of a task. Fix a summary $S=AH$ and stratum weights $\omega_k\ge0$ with $\sum_k\omega_k=1$. Write $\Cov_k$ for a covariance in stratum $k$ and $\gamma_k$ for its mean gradient. Put
\[
\bar\Sigma_S\coloneqq\sum_k\omega_k\Cov_k(S),\qquad\bar\sigma_{SY}\coloneqq\sum_k\omega_k\Cov_k(S,Y),\qquad\bar\sigma_{S\gamma}\coloneqq\sum_k\omega_k\Cov_k(S,H)\gamma_k,
\]
and assume $\bar\Sigma_S\succ0$. Pooled degrees of freedom alone do not guarantee this condition, which is therefore a condition of every pooled analysis. For an intervention with mean edit $\bar u$ the pooled probe implies the change $\rho\coloneqq\xi^\top\bar\sigma_{SY}$ with $\xi\coloneqq\bar\Sigma_S^{-1}A\bar u$. Adding and subtracting $\xi^\top\bar\sigma_{S\gamma}$ gives $\rho-\alpha_{\rm ref}=\xi^\top(\bar\sigma_{SY}-\bar\sigma_{S\gamma})+(\xi^\top\bar\sigma_{S\gamma}-\alpha_{\rm ref})$ for every reference attribution $\alpha_{\rm ref}$. The first term is the pooled Stein residual and the second the pooled omitted-variable term. For a scalar score and a mean edit $\bar u$ common to all strata, as in steering, the two terms are averages of the stratum-specific terms of \eqref{eq:gap}. Write $\sigma_k^2$ for the variance of the score in stratum $k$ and put $q_k\coloneqq\omega_k\sigma_k^2/\sum_l\omega_l\sigma_l^2$, which vanishes for strata with $\sigma_k^2=0$. Every sum below runs over the strata with $\sigma_k^2>0$. Since $\Cov_k(S,Y)=\sigma_k^2\beta_k$ for the stratum-specific coefficient $\beta_k$,
\begin{align*}
\beta_{\rm pool}&\coloneqq\bar\sigma_{SY}/\bar\Sigma_S=\textstyle\sum_kq_k\beta_k,\\
\xi^\top(\bar\sigma_{SY}-\bar\sigma_{S\gamma})&=\textstyle\sum_kq_k\,\st^{(k)}_{\bar u'_k},\\
\xi^\top\bar\sigma_{S\gamma}-\textstyle\sum_kq_k\bar u^\top\gamma_k&=\textstyle\sum_kq_k\,\cm_k(\Delta),
\end{align*}
where $\st^{(k)}$ and $\cm_k$ are the terms of \eqref{eq:gap} in stratum $k$ and $\bar u'_k\coloneqq A^\top A\bar u/\sigma_k^2$. The reference attribution of the fixed-direction analysis is accordingly $\alpha_{\rm ref}=\sum_kq_k\bar u^\top\gamma_k$, the attribution averaged with the weights of the probe.

The identity also covers the statistics of published SAE pipelines.\label{app:survey} For a scalar summary, a correlation determines $\Cov(S,Y)$ up to scale, as does a group-mean difference for a binary target. \Cref{thm:gap} applies once these quantities are converted into the probe-implied change for the specified intervention, with the probe refitted when the target changes from correctness to the log-probability. The gap then consists of a Stein residual and an omitted-variable term for fixed-size steering and also includes the edit-gradient covariance for ablation or clamping. With a model-derived target and one backward pass per prompt, the attribution and these terms are estimable for pipelines that pass from a rank test, a $t$-statistic or a correlation to an ablation or a steering intervention \citep{patel2026functional,tahimic2025mechanistic,cho2026corrsteer} and for those that rank latents or control the errors of a selection \citep{zhang2025saefire,enkhbayar2025knockoffs,carlson2025making}.

\subsection{Stein residual and omitted-variable term}\label{app:terms}

When $\Sigma_H\succ0$, the least-squares slope of \Cref{lem:range} is unique and equals $\beta_\star=\Sigma_H^{-1}\EE[(H-\mu)Y]=\EE[Y\,s_0(H)]$ for the score $s_0(x)\coloneqq\Sigma_H^{-1}(x-\mu)$ of the Gaussian law with the activation's moments. A linear probe is then an average-derivative estimate under the Gaussian score, whose Stein residual vanishes for Gaussian activations (\Cref{lem:stein}\ref{lem:stein:act}). The residual also vanishes at every law when $F$ is affine with gradient $g$, since $\EE[\eta\mid H]=0$ then gives $\Cov(H,Y)=\Cov(H,F(H))=\Sigma_Hg$ and $\gamma=g$. The empirical law of a stratum is finitely supported and hence not Gaussian. Its covariance is singular whenever the stratum has at most $p$ prompts (\Cref{lem:range}). A support of affinely independent points, as for the empirical law of a stratum with affinely independent activations, has at most $p+1$ points and gives the Stein residual a secant form.

\begin{lemma}[Secant form of the Stein residual]\label{lem:finite-stein}
Let $\operatorname{supp}(H)=\{x_1,\ldots,x_{n_0}\}$ consist of affinely independent points. Then a unique pair $(b,b_0)\in V\times\RR$ satisfies $F(x_i)=b_0+b^\top x_i$ for every $i$. If $\EE[Y\mid H]=F(H)$, then $\Cov(H,Y)=\Sigma_Hb$ and
\begin{align*}
\st_w&=(\Sigma_Hw)^\top(b-\gamma)\\
&=(\Sigma_Hw)^\top(b-\gamma_V)\qquad\text{for every }w\in\RR^p .
\end{align*}
If $F$ is affine, then $b=\gamma_V$ and every Stein residual vanishes.
\end{lemma}

\begin{proof}[Proof of \Cref{lem:finite-stein}]
Affine independence makes the vectors $x_k-x_1$ with $k\ge2$ a basis of $V$. Every interpolating pair $(b,b_0)\in V\times\RR$ satisfies $b^\top(x_k-x_1)=F(x_k)-F(x_1)$ for $k\ge2$. These equations prescribe the inner products of $b$ with a basis of $V$ and therefore have a unique solution $b\in V$. The intercept is then fixed as $b_0\coloneqq F(x_1)-b^\top x_1$. This pair interpolates $F$ at every support point. Since $F(H)=b_0+b^\top H$ almost surely, $\Cov(H,F(H))=\Sigma_Hb$. Because $\EE[Y\mid H]=F(H)$, the tower property gives $\Cov(H,Y)=\EE[(H-\mu)F(H)]=\Sigma_Hb$. Substituting this into \eqref{eq:steinres} yields
\begin{align*}
\st_w
&=w^\top\Sigma_Hb-w^\top\Sigma_H\gamma\\
&=(\Sigma_Hw)^\top(b-\gamma)\\
&=(\Sigma_Hw)^\top(b-\gamma_V).
\end{align*}
The last step follows from $\Sigma_Hw\in\range(\Sigma_H)=V$ (\Cref{lem:range}), which annihilates $\gamma_{V^\perp}$. Finally, let $F$ be affine with gradient $g$. Since $x_k-x_1\in V$, $F(x_k)-F(x_1)=g^\top(x_k-x_1)=g_V^\top(x_k-x_1)$. Uniqueness then gives $b=g_V$. The mean gradient is $\gamma=g$, which gives $b=\gamma_V$. The residual therefore vanishes. This completes the proof.
\end{proof}

The model-derived targets used here give $\EE[Y\mid H]=F(H)$ at the empirical law of a stratum. The plug-in Stein residual of a stratum with affinely independent activations therefore takes the secant form \eqref{eq:secant}, which covers every assessed evaluation stratum except one IOI stratum per model (\Cref{app:span-check}). The Stein residual measured in \Cref{app:e1r} is a within-stratum U-statistic, which for the unnormalized residual $\st_w$ in a fixed direction $w$ is unbiased for the residual of the population law of the stratum.

For a single read direction fix $d$ with $\|d\|=1$ and put $A\coloneqq d^\top E_{\rm fin}$ and $u\coloneqq E_{\rm fin}^\top d$. Write $\sigma^2\coloneqq d^\top\Cov(h)d=u^\top\Sigma_Hu$ for the variance of the score $d^\top h$ and $\|x\|_M\coloneqq(x^\top Mx)^{1/2}$ for positive definite $M$. When $\Sigma_H\succ0$, the variance inflation of the score is $\mathrm{VIF}\coloneqq\sigma^2u^\top\Sigma_H^{-1}u$.

\begin{lemma}[Single-direction geometry]\label{lem:vif}
Let $\Sigma_H\succ0$. Then $\sigma^2>0$, $v_A(u)=\Sigma_Hu/\sigma^2$, $u'=u/\sigma^2$ and $\st_{u'}=\st_u/\sigma^2$. The variance inflation satisfies $\mathrm{VIF}\ge1$ and
\[
\|u-v_A(u)\|^2_{\Sigma_H^{-1}}=\frac{\mathrm{VIF}-1}{\sigma^2}.
\]
Moreover, $\mathrm{VIF}=1$ if and only if the score $d^\top h$ is uncorrelated with every $w^\top H$ with $w\perp u$. In that case $\Sigma_Hu=\sigma^2u$ and $v_A(u)=u$. The omitted-variable term then vanishes for every edit with $\bar u\in\operatorname{span}\{u\}$.
\end{lemma}

\begin{proof}[Proof of \Cref{lem:vif}]
The proof proceeds in two parts. First, we compute the normalizing quantities and derive the distance identity. Second, we prove $\mathrm{VIF}\ge1$ and characterize equality. Since the position blocks are disjoint, $E_{\rm fin}E_{\rm fin}^\top$ is the identity.

\textbf{Part 1 (Normalization and distance).}
Since $A^\top=u$ and $\|u\|=\|d\|=1$, we have $Au=d^\top E_{\rm fin}E_{\rm fin}^\top d=1$ and $A\Sigma_HA^\top=u^\top\Sigma_Hu=\sigma^2$, which is positive because $\Sigma_H\succ0$. Substituting these scalars into the definitions gives $v_A(u)=\Sigma_Hu/\sigma^2$ and $u'=u/\sigma^2$. The Stein residual is linear in its direction, which gives $\st_{u'}=\st_u/\sigma^2$. Next, $\|v_A(u)\|^2_{\Sigma_H^{-1}}=u^\top\Sigma_Hu/\sigma^4=1/\sigma^2$ and $v_A(u)^\top\Sigma_H^{-1}\{u-v_A(u)\}=u^\top u/\sigma^2-1/\sigma^2=0$. The Pythagorean identity in the $\Sigma_H^{-1}$ inner product therefore yields
\begin{align*}
\|u-v_A(u)\|^2_{\Sigma_H^{-1}}
&=\|u\|^2_{\Sigma_H^{-1}}-\|v_A(u)\|^2_{\Sigma_H^{-1}}\\
&=u^\top\Sigma_H^{-1}u-\frac1{\sigma^2}\\
&=\frac{\mathrm{VIF}-1}{\sigma^2}.
\end{align*}
The last step follows from the definition of $\mathrm{VIF}$.

\textbf{Part 2 (Variance inflation and its equality case).}
The Cauchy-Schwarz inequality for $\Sigma_H^{1/2}u$ and $\Sigma_H^{-1/2}u$ gives $1=(u^\top u)^2\le(u^\top\Sigma_Hu)(u^\top\Sigma_H^{-1}u)=\mathrm{VIF}$. Equality holds if and only if $\Sigma_H^{1/2}u$ is parallel to $\Sigma_H^{-1/2}u$, that is, if and only if $\Sigma_Hu$ is parallel to $u$. Since the vectors orthogonal to $u^\perp$ form $\operatorname{span}\{u\}$, this holds if and only if $\Cov(d^\top h,w^\top H)=w^\top\Sigma_Hu=0$ for every $w\perp u$. In that case $\Sigma_Hu=(u^\top\Sigma_Hu)\,u=\sigma^2u$, which gives $v_A(u)=u$. For an edit with $\bar u=\vartheta u$ the linearity of $v_A$ then gives $v_A(\bar u)=\bar u$ and hence $\cm(\Delta)=(v_A(\bar u)-\bar u)^\top\gamma=0$. This completes the proof.
\end{proof}

The probe's pattern $v_A(u)=\Cov(H,d^\top h)/\VV(d^\top h)$ collects the slopes of the coordinates of $H$ on the score. It therefore also moves the coordinates correlated with the score, even at unread positions. For an edit $\bar u=\vartheta u$ along the read direction, $\cm(\Delta)=\vartheta\,(v_A(u)-u)^\top\gamma$ and the variance inflation measures the displacement $u-v_A(u)$ (\Cref{lem:vif}). The hypothesis $\Sigma_H\succ0$ fails at the empirical law of a stratum with at most $p$ prompts (\Cref{lem:range}). The variance inflation of \Cref{app:e1r} is therefore a pooled selection-split proxy for the dependence among latent scores. It does not bound the omitted-variable term of a stratum.

\paragraph{Closed-form examples.} \Cref{tab:toy} gives one example per term of \eqref{eq:gap}, each with a scalar or two-coordinate activation, $Y=F(H)$ and a probe on the first coordinate. In the first example the probe's secant slope $1$ differs from the mean gradient $3$. In the second the probe's pattern $v_A(e_1)=(1,1)$ moves a coordinate that the edit leaves fixed. In the third the edit $-H$ and the gradient $H/2$ vary together across prompts, which gives $M=C=-\frac12$ in \eqref{eq:ablate}.

\begin{table}[h]
\centering\footnotesize
\setlength{\tabcolsep}{4pt}
\begin{tabular}{@{}lllcccccc@{}}
\toprule
Law of $H$ & $F(x)$ & Edit & $\rho$ & $\alpha$ & $\st_{\bar u'}$ & $\cm$ & $\ec$ \\
\midrule
uniform on $\{-1,1\}$ & $x^3$ & $1$ & $1$ & $3$ & $-2$ & $0$ & $0$ \\
$(X,X)$, $\VV(X)>0$ & $x_2$ & $e_1$ & $1$ & $0$ & $0$ & $1$ & $0$ \\
uniform on $\{0,2\}$ & $x^2/4$ & $-H$ & $-\frac12$ & $-1$ & $0$ & $0$ & $-\frac12$ \\
\bottomrule
\end{tabular}
\caption{Closed-form examples of the three terms of \eqref{eq:gap}, with $A=1$ or $A=(1,0)$. Each row satisfies $\rho-\alpha=\st_{\bar u'}+\cm-\ec$.}\label{tab:toy}
\end{table}

\subsection{Proof of \texorpdfstring{\Cref{cor:ablate}}{Corollary 1}}\label{app:proofs:ablate}

We derive \Cref{cor:ablate} from the following decomposition for any edit along one direction $d$ at a set $\mathcal T$ of positions with prompt-dependent coefficients $\delta_\tau$.

\begin{corollary}[Ablations and edits at unread positions]\label{cor:ablate-general}
Assume the hypotheses of \Cref{thm:gap}. Let $\Delta(H)=\sum_{\tau\in\mathcal T}\delta_\tau(H)E_\tau^\top d$ for a direction $d$, a set $\mathcal T$ of edited positions and coefficients with $\EE\delta_\tau^2<\infty$. Put $\Gamma_\tau\coloneqq d^\top E_\tau\nabla F$.
\begin{enumerate}[label=(\alph*),leftmargin=*,itemsep=0pt]
\item\label{cor:ablate-general:dec} $\alpha(\Delta)=\sum_\tau\EE[\delta_\tau]\EE[\Gamma_\tau]+\ec(\Delta)$ with $\ec(\Delta)=\sum_\tau\Cov(\delta_\tau,\Gamma_\tau)$. For the zero-ablation of latent $j$ ($d=d_j$, $\delta_\tau=-z_{\tau j}$),
\[
\alpha(\Delta_j)=L_j\coloneqq M_j+C_j,
\]
where $M_j\coloneqq-\sum_\tau\EE[z_{\tau j}]\EE[\Gamma_\tau]$ and $C_j\coloneqq-\sum_\tau\Cov(z_{\tau j},\Gamma_\tau)$. Mean-ablation to fixed references has the same $C_j$.
\item\label{cor:ablate-general:blind} (Edits at unread positions) If $A=\sum_{\tau\in\mathcal T_A}A_\tau E_\tau$ and $\delta_\tau=0$ almost surely for every $\tau\in\mathcal T_A$, then $\rho(\Delta)=\st_{\bar u'}=0$ and $v_A(\bar u)=0$. For the zero-ablation of a ReLU or JumpReLU latent with $A=d_j^\top E_{\rm fin}$ and ${\rm fin}\in\mathcal T$, the mean $A\bar u=-\EE z_{{\rm fin},j}$ vanishes if and only if the latent never fires at the final position. In that case the conclusions above hold with $\mathcal T_A=\{{\rm fin}\}$.
\end{enumerate}
\end{corollary}

\begin{proof}[Proof of \Cref{cor:ablate-general}]
The proof proceeds in two parts. First, we write the expansion of $\alpha$ from Part 1 of the proof of \Cref{thm:gap} position by position, which gives \ref{cor:ablate-general:dec}. Second, we show that the probe sees the edit through $A\bar u$ alone, which gives \ref{cor:ablate-general:blind}. Throughout, $\delta_\tau\coloneqq0$ for $\tau\notin\mathcal T$. Each position $\tau$ after the first indexes a coordinate block $E_\tau H$ of $H$. The blocks are disjoint, which makes $E_\tau E_\sigma^\top$ the identity for $\sigma=\tau$ and zero otherwise.

\textbf{Part 1 (Decomposition of the attribution).}
Linearity of expectation gives
\[
\bar u=\sum_\tau\EE[\delta_\tau]\,E_\tau^\top d,\qquad\Delta(H)-\bar u=\sum_\tau(\delta_\tau-\EE\delta_\tau)\,E_\tau^\top d .
\]
Since $(E_\tau^\top d)^\top\nabla F=\Gamma_\tau$, the first identity gives $\bar u^\top\gamma=\sum_\tau\EE[\delta_\tau]\,\EE[\Gamma_\tau]$. The second identity yields
\begin{align*}
\ec(\Delta)
&=\sum_\tau\EE\bigl[(\delta_\tau-\EE\delta_\tau)\,(E_\tau^\top d)^\top\bigl(\nabla F(H)-\gamma\bigr)\bigr]\\
&=\sum_\tau\EE\bigl[(\delta_\tau-\EE\delta_\tau)(\Gamma_\tau-\EE\Gamma_\tau)\bigr]\\
&=\sum_\tau\Cov(\delta_\tau,\Gamma_\tau).
\end{align*}
The second step uses $(E_\tau^\top d)^\top\gamma=\EE\Gamma_\tau$. The last step follows from the definition of a covariance. Every term is finite because $\EE\delta_\tau^2<\infty$ and $\EE\Gamma_\tau^2\le\|d\|^2\EE\|\nabla F\|^2$. Substituting both sums into $\alpha(\Delta)=\bar u^\top\gamma+\ec(\Delta)$ gives the decomposition of part~\ref{cor:ablate-general:dec}. Zero-ablation of latent $j$ has $d=d_j$ and $\delta_\tau=-z_{\tau j}$, which gives $\alpha(\Delta_j)=M_j+C_j$. Mean-ablation to fixed references has $\delta_\tau=r_{\tau j}-z_{\tau j}$ and the same $C_j$, since $\Cov(r_{\tau j}-z_{\tau j},\Gamma_\tau)=-\Cov(z_{\tau j},\Gamma_\tau)$.

\textbf{Part 2 (Edits at unread positions).}
For $\tau\in\mathcal T_A$ the hypothesis $\delta_\tau=0$ almost surely gives $\EE\delta_\tau=0$. Block orthogonality then gives $E_\tau\bar u=\EE[\delta_\tau]\,d=0$. Hence $A\bar u=0$ and $\bar u'=A^\top(A\Sigma_HA^\top)^{-1}A\bar u=0$. The Stein residual is linear in its direction, which gives $\st_{\bar u'}=0$. The identities $v_A(\bar u)=\Sigma_H\bar u'$ and $\rho(\Delta)=\bar u'^\top\Cov(H,Y)$ from Part 2 of the proof of \Cref{thm:gap} give $v_A(\bar u)=0$ and $\rho(\Delta)=0$. The omitted-variable term is then $\cm(\Delta)=-\bar u^\top\gamma$. Since $\delta_\tau=0$ almost surely for $\tau\in\mathcal T_A$, the summands of $\bar u^\top\gamma$ and $\ec(\Delta)$ at these positions vanish. The whole attribution is therefore carried by the positions in $\mathcal T\setminus\mathcal T_A$. Now let $\Delta$ be the edit of zero-ablating a ReLU or JumpReLU latent $j$ with ${\rm fin}\in\mathcal T$ and let $A=d_j^\top E_{\rm fin}$. Then $\mathcal T_A=\{{\rm fin}\}$. Applying $A$ to the expansion of $\bar u$ in Part 1 yields
\begin{align*}
A\bar u
&=d_j^\top E_{\rm fin}\sum_{\tau\in\mathcal T}\EE[\delta_\tau]\,E_\tau^\top d_j\\
&=-\EE[z_{{\rm fin},j}]\,d_j^\top d_j\\
&=-\EE z_{{\rm fin},j}.
\end{align*}
The second step uses ${\rm fin}\in\mathcal T$ and block orthogonality. The last step follows from $\|d_j\|=1$. A ReLU or JumpReLU code is nonnegative. Hence $\EE z_{{\rm fin},j}=0$ if and only if $z_{{\rm fin},j}=0$ almost surely, that is, if and only if the latent never fires at the final position. In that case $\delta_{\rm fin}=0$ almost surely, which is the hypothesis of part~\ref{cor:ablate-general:blind} with $\mathcal T_A=\{{\rm fin}\}$. Part~\ref{cor:ablate-general:dec} uses the displayed moments alone and part~\ref{cor:ablate-general:blind} uses them together with $A\Sigma_HA^\top\succ0$. Both parts therefore hold for the empirical law of each stratum in which $A\Sigma_HA^\top\succ0$. This completes the proof.
\end{proof}

\noindent\textbf{Size of the covariance.} For an ablation at one position with nonzero means and positive variances, $|C_j|/|M_j|=|\operatorname{corr}(z_{\tau j},\Gamma_\tau)|\,\mathrm{CV}(z_{\tau j})\,\mathrm{CV}(\Gamma_\tau)$, where $\mathrm{CV}(X)\coloneqq\VV(X)^{1/2}/|\EE X|$ is the coefficient of variation. This follows from $C_j=-\operatorname{corr}(z_{\tau j},\Gamma_\tau)\VV(z_{\tau j})^{1/2}\VV(\Gamma_\tau)^{1/2}$ and $M_j=-\EE[z_{\tau j}]\EE[\Gamma_\tau]$. Rare firing makes the code's coefficient of variation large and can make $|C_j|$ exceed $|M_j|$.

\begin{proof}[Proof of \Cref{cor:ablate}]
Zero-ablation applies the edit $\Delta_j(H)=-\sum_\tau z_{\tau j}E_\tau^\top d_j$ at every position after the first. It is the edit of \Cref{cor:ablate-general} with $d=d_j$, $\delta_\tau=-z_{\tau j}$ and $\mathcal T$ the set of all positions but the first. Since the blocks are disjoint and $\|d_j\|=1$, $\|\Delta_j(H)\|^2=\sum_\tau z_{\tau j}^2$. The hypothesis $\EE\|\Delta_j(H)\|^2<\infty$ of \Cref{thm:gap} therefore gives $\EE z_{\tau j}^2<\infty$ for every $\tau$. Part~\ref{cor:ablate:dec} is then the zero-ablation case of \Cref{cor:ablate-general}\ref{cor:ablate-general:dec}. For part~\ref{cor:ablate:blind}, the summary $A=d_j^\top E_{\rm fin}$ reads $\mathcal T_A=\{{\rm fin}\}$. If $z_{{\rm fin},j}=0$ almost surely, then $\delta_{\rm fin}=0$ almost surely. \Cref{cor:ablate-general}\ref{cor:ablate-general:blind} then gives $\rho(\Delta_j)=0$. Hence $\rho(\Delta_j)-\alpha(\Delta_j)=-\alpha(\Delta_j)$. This completes the proof.
\end{proof}

%% file: sections/app_inference.tex
\section{Inference for the reported intervals}\label{app:inference}

The terms of \eqref{eq:gap} admit classical inference because the selection split fixes every summary, direction and edit that enters them. This appendix states that design, proves the many-strata inference and assesses coverage by simulation.

\begin{assumption}[Split design]\label{asm:split}
\leavevmode
\begin{enumerate}[label=(\roman*),leftmargin=1.6em,itemsep=0pt]
\item\label{asm:split:ind} (Independence) The SAE, the active set, the selections, the summaries and the edits are fixed by data independent of the evaluation split \citep{cox1975note,rinaldo2019bootstrapping}. The evaluation prompts are independent draws from each task's law.
\item\label{asm:split:mom} (Moments) Within each evaluation stratum the activations, the targets and the gradients have finite second moments.
\end{enumerate}
\end{assumption}

Item~(i) makes the samples of different evaluation strata independent given their answers and lengths. Item~(ii), with the moment conditions of \Cref{thm:gap}, makes every moment in \eqref{eq:gap} finite. A reported quantity averages within-stratum estimates over many small strata (\Cref{app:protocols}). The following theorem, whose full statement is \Cref{prop:cells}, justifies its stratum-level interval.

\begin{theorem}[Inference over many small strata]\label{thm:inference}
Assume \Cref{asm:cut,asm:split}. Let the samples be independent across strata, let each stratum quantity have an unbiased estimate from its own stratum, let the largest stratum weight tend to zero and let a fourth-moment Lyapunov condition hold. As the number of strata grows, the stratum-level interval for a fixed-weight average of stratum quantities has asymptotic coverage at least its level, with equality in the limit when the stratum quantities agree. Under the further conditions of \Cref{prop:cells}, which include a denominator bounded away from zero, the same holds for a ratio of two such averages.
\end{theorem}

The intervals target the fixed-weight average of the stratum targets and are therefore asymptotically conservative when the stratum targets differ. At the observed stratum counts they rely on a finite-sample normal approximation. Under the sparse JumpReLU design at the observed stratum sizes, simulated intervals cover at $0.919$ to $0.942$ for the ablation terms and at $0.820$ to $0.835$ for the Stein residual (\Cref{app:coverage}). The theorem concerns the stratum-level $t$ intervals and the per-pair tests built from them. The percentile bootstrap intervals resample strata or tasks as specified in each analysis. \Cref{prop:cells} does not establish coverage for these bootstrap intervals, including those for medians and shares over pairs.

For strata $k=1,\ldots,K_s$ let $\hat\theta_k$ be an estimate of a stratum target $\theta_k$ computed from stratum $k$ alone, with variance $v_k\coloneqq\VV\hat\theta_k$. Given weights $\omega_k\ge0$ fixed by the stratum sizes with $\sum_k\omega_k=1$, we define
\[
\begin{gathered}
\bar\theta\coloneqq\sum_k\omega_k\theta_k,\qquad\hat\theta\coloneqq\sum_k\omega_k\hat\theta_k,\qquad\bar\omega\coloneqq\max_k\omega_k,\qquad\varsigma^2\coloneqq\sum_k\omega_k^2v_k,\\
\mathsf h^2\coloneqq\sum_k\omega_k^2(\theta_k-\bar\theta)^2,\qquad\hat\varsigma^2\coloneqq\sum_k\omega_k^2(\hat\theta_k-\hat\theta)^2,\qquad\Upsilon_4\coloneqq\varsigma^{-4}\sum_k\omega_k^4\,\EE(\hat\theta_k-\theta_k)^4,
\end{gathered}
\]
where $\varsigma^2$ is the variance of $\hat\theta$ for estimates independent across strata, $\mathsf h^2$ measures the heterogeneity of the stratum targets and $\Upsilon_4$ is a Lyapunov ratio with exponent four. For $0<a<1$ we write $z_{1-a/2}\coloneqq\Phi^{-1}(1-a/2)$, with $\Phi$ the standard normal distribution function. For a ratio, each stratum carries estimates $\hat N_k$ and $\hat Q_k$ of targets $N_k$ and $Q_k$, averaged with the same weights into $\bar N$, $\hat N$, $\bar Q$ and $\hat Q$. The ratio $\bar R\coloneqq\bar N/\bar Q$ has the estimate $\hat R\coloneqq\hat N/\hat Q$. The linearized stratum estimates $\hat y_k\coloneqq\hat N_k-\bar R\hat Q_k$ have targets $y_k\coloneqq N_k-\bar RQ_k$ with average $\bar y=\bar N-\bar R\bar Q=0$. We write $\varsigma_y^2$, $\mathsf h_y^2$ and $\Upsilon_{4,y}$ for the quantities above computed from $(\hat y_k,y_k)$ and $\varsigma_Q^2$ and $\mathsf h_Q^2$ for those computed from $(\hat Q_k,Q_k)$.

\begin{proposition}[Inference over many small strata]\label{prop:cells}
Fix the selection split and the stratum sizes. Let the stratum estimates be independent with $\EE\hat\theta_k=\theta_k$ and $v_k<\infty$ for every $k$. Suppose that $\varsigma^2>0$, $\bar\omega\to0$ and $\Upsilon_4\to0$ as $K_s\to\infty$.
\begin{enumerate}[label=(\alph*),leftmargin=*,itemsep=0pt]
\item\label{prop:cells:clt} (Normality) $(\hat\theta-\bar\theta)/\varsigma\to\cN(0,1)$ in law.
\item\label{prop:cells:var} (Variance estimate) $\hat\varsigma^2/(\varsigma^2+\mathsf h^2)\to1$ in probability.
\item\label{prop:cells:cov} (Coverage) For every $a\in(0,1)$, $\liminf_{K_s\to\infty}\PP(|\hat\theta-\bar\theta|\le z_{1-a/2}\hat\varsigma)\ge1-a$. If $\mathsf h^2/\varsigma^2\to0$, this coverage converges to $1-a$.
\item\label{prop:cells:ratio} (Ratios) Let instead the stratum estimates be pairs $(\hat N_k,\hat Q_k)$ with means $(N_k,Q_k)$ and finite variances. Keep $\bar\omega\to0$ and, in place of the conditions on $\varsigma^2$ and $\Upsilon_4$, suppose that $|\bar Q|$ is bounded away from zero, $\varsigma_Q^2+\mathsf h_Q^2\to0$, $\varsigma_y^2>0$ and $\Upsilon_{4,y}\to0$. Then the event $\hat Q\ne0$ has probability tending to one and $\bar Q(\hat R-\bar R)/\varsigma_y\to\cN(0,1)$ in law. For every $a\in(0,1)$ the interval $\hat R\pm z_{1-a/2}\hat\varsigma_R$ with the linearized stratum-level variance $\hat\varsigma_R^2\coloneqq\sum_k\omega_k^2(\hat N_k-\hat R\hat Q_k)^2/\hat Q^2$ has asymptotic coverage at least $1-a$. Its coverage converges to $1-a$ if $\mathsf h_y^2/\varsigma_y^2\to0$.
\end{enumerate}
\end{proposition}

The within-stratum U-statistics of the fixed-direction analysis and of the ablation terms, which need $n_k\ge3$, are stratum estimates of this kind, pooled with weights $\omega_k\propto n_k-1$. The interval used in the analyses multiplies $\hat\varsigma^2$ by $K_s/(K_s-1)$ and replaces $z_{1-a/2}$ by the $t_{K_s-1}$ quantile $t_{K_s-1,1-a/2}$. Its half-width is that of \Cref{prop:cells}(c) times $\{K_s/(K_s-1)\}^{1/2}t_{K_s-1,1-a/2}/z_{1-a/2}$, which exceeds one because $K_s/(K_s-1)>1$ and $t_{K_s-1,1-a/2}>z_{1-a/2}$ and tends to one as $K_s\to\infty$. This interval therefore contains the interval of (c), inherits its coverage bound and by Slutsky's lemma its equality case. The same holds for the interval of (d).

\subsection{Proof of \texorpdfstring{\Cref{prop:cells}}{Proposition C.1}}\label{app:cells-proof}

\begin{proof}[Proof of \Cref{prop:cells}]
Write $e_k\coloneqq\hat\theta_k-\theta_k$ for the stratum errors, which are independent with mean zero, $d_k\coloneqq\theta_k-\bar\theta$ for the stratum deviations and $\bar e\coloneqq\hat\theta-\bar\theta=\sum_k\omega_ke_k$ for the error of the average.

\textbf{Part 1 (Asymptotic normality).}
The summands $\omega_ke_k$ of $\bar e$ form a triangular array of independent variables with mean zero and variances summing to $\varsigma^2>0$. Their Lyapunov ratio with exponent four, $\varsigma^{-4}\sum_k\EE(\omega_ke_k)^4=\varsigma^{-4}\sum_k\omega_k^4\EE e_k^4=\Upsilon_4$, tends to zero. Lyapunov's theorem therefore gives $\bar e/\varsigma\to\cN(0,1)$ in law, which is (a) because $\bar e=\hat\theta-\bar\theta$.

\textbf{Part 2 (Variance estimate).}
We show that $\hat\varsigma^2-(\varsigma^2+\mathsf h^2)=o_p(\varsigma^2+\mathsf h^2)$. Since $\hat\theta_k-\hat\theta=e_k+d_k-\bar e$ and $\sum_k\omega_k^2d_k^2=\mathsf h^2$, expanding the square in $\hat\varsigma^2$ yields
\[
\hat\varsigma^2-(\varsigma^2+\mathsf h^2)
=\underbrace{\sum_k\omega_k^2e_k^2-\varsigma^2}_{O_p(\Upsilon_4^{1/2}\varsigma^2)}
+\underbrace{2\sum_k\omega_k^2e_kd_k}_{O_p(\Upsilon_4^{1/4}\varsigma\mathsf h)}
-\underbrace{2\bar e\sum_k\omega_k^2(e_k+d_k)}_{O_p(\bar\omega\varsigma^2+\bar\omega^{1/2}\varsigma\mathsf h)}
+\underbrace{\bar e^2\sum_k\omega_k^2}_{O_p(\bar\omega\varsigma^2)} .
\]
Independence across strata gives the variance bounds
\begin{align*}
\VV\Bigl(\sum_k\omega_k^2e_k^2\Bigr)&\le\sum_k\omega_k^4\EE e_k^4=\Upsilon_4\varsigma^4,\\
\VV\Bigl(\sum_k\omega_k^2e_kd_k\Bigr)&=\sum_k\omega_k^4d_k^2v_k\le\max_k(\omega_k^2v_k)\,\mathsf h^2\le\Upsilon_4^{1/2}\varsigma^2\mathsf h^2,\\
\VV\Bigl(\sum_k\omega_k^2e_k\Bigr)&=\sum_k\omega_k^4v_k\le\bar\omega^2\varsigma^2,
\end{align*}
where the last step of the second line uses $\omega_k^2v_k\le(\omega_k^4\EE e_k^4)^{1/2}\le\Upsilon_4^{1/2}\varsigma^2$, which follows from Jensen's inequality. The first sum has mean $\varsigma^2$ and the second term has mean zero. Chebyshev's inequality and the first two bounds therefore give the orders of the first two terms. For the third term, $\bar e=O_p(\varsigma)$ because $\VV\bar e=\varsigma^2$. Chebyshev's inequality and the third bound give $\sum_k\omega_k^2e_k=O_p(\bar\omega\varsigma)$. The Cauchy-Schwarz inequality with $\sum_k\omega_k^2\le\bar\omega\sum_k\omega_k=\bar\omega$ gives $|\sum_k\omega_k^2d_k|\le\bar\omega^{1/2}\mathsf h$. The fourth term is at most $\bar\omega\bar e^2=O_p(\bar\omega\varsigma^2)$. Since $\Upsilon_4\to0$, $\bar\omega\to0$ and $\varsigma\mathsf h\le(\varsigma^2+\mathsf h^2)/2$, each of the four orders is $o(\varsigma^2+\mathsf h^2)$. Dividing by $\varsigma^2+\mathsf h^2>0$ proves (b).

\textbf{Part 3 (Coverage).}
Put $c_{K_s}\coloneqq\varsigma/(\varsigma^2+\mathsf h^2)^{1/2}$, which lies in $(0,1]$ because $\varsigma^2>0$. By (b) the event $\hat\varsigma>0$ has probability tending to one. On this event $(\hat\theta-\bar\theta)/\hat\varsigma=\{(\hat\theta-\bar\theta)/\varsigma\}\,c_{K_s}\,(\varsigma^2+\mathsf h^2)^{1/2}/\hat\varsigma$. By (a) the first factor converges in law to $\cN(0,1)$. By (b) and the continuity of $x\mapsto x^{-1/2}$ at one, the last factor converges to one in probability. Every subsequence has a further subsequence along which $c_{K_s}$ converges to some $c\in[0,1]$. Along it Slutsky's lemma gives $(\hat\theta-\bar\theta)/\hat\varsigma\to\cN(0,c^2)$ in law. The coverage therefore converges to $2\Phi(z_{1-a/2}/c)-1\ge2\Phi(z_{1-a/2})-1=1-a$ if $c>0$ and, because $z_{1-a/2}>0$, to one if $c=0$. Since every subsequence thus has a further subsequence with limiting coverage at least $1-a$, the limit inferior of the coverage is at least $1-a$. If $\mathsf h^2/\varsigma^2\to0$, then $c_{K_s}\to1$ along the whole sequence and the coverage converges to $2\Phi(z_{1-a/2})-1=1-a$, which proves (c).

\textbf{Part 4 (Ratios).}
The average $\hat Q$ has mean $\bar Q$ and variance $\varsigma_Q^2\to0$. Chebyshev's inequality gives $\hat Q-\bar Q=o_p(1)$. Since $|\bar Q|$ is bounded away from zero, $\hat Q/\bar Q\to1$ in probability and the event $\hat Q\ne0$ has probability tending to one. On this event $\hat R-\bar R=(\hat N-\bar R\hat Q)/\hat Q=\hat y/\hat Q$ with $\hat y\coloneqq\sum_k\omega_k\hat y_k=\hat N-\bar R\hat Q$. The $\hat y_k$ are independent across strata with means $y_k$ and finite variances. With $\varsigma_y^2>0$, $\bar\omega\to0$ and $\Upsilon_{4,y}\to0$, Parts 1 and 2 apply to them. Since $\bar y=0$, they give $\hat y/\varsigma_y\to\cN(0,1)$ in law and $\hat\varsigma_y^2/(\varsigma_y^2+\mathsf h_y^2)\to1$ in probability, with $\hat\varsigma_y^2\coloneqq\sum_k\omega_k^2(\hat y_k-\hat y)^2$. Since $\bar Q(\hat R-\bar R)/\varsigma_y=(\hat y/\varsigma_y)(\bar Q/\hat Q)$, Slutsky's lemma gives the normal limit.

For the variance estimate put $\hat x_k\coloneqq\hat N_k-\hat R\hat Q_k$, $\hat s^2\coloneqq\sum_k\omega_k^2\hat x_k^2$ and $W_Q\coloneqq\sum_k\omega_k^2(\hat Q_k-\hat Q)^2$. Then $\hat\varsigma_R^2=\hat s^2/\hat Q^2$. Since $\hat y=(\hat R-\bar R)\hat Q$, each $\hat x_k$ satisfies $\hat x_k=\hat y_k-(\hat R-\bar R)\hat Q_k=(\hat y_k-\hat y)-(\hat R-\bar R)(\hat Q_k-\hat Q)$. Expanding the square and applying the Cauchy-Schwarz inequality to the cross term yields
\begin{align*}
|\hat s^2-\hat\varsigma_y^2|\le2|\hat R-\bar R|\,\hat\varsigma_yW_Q^{1/2}+(\hat R-\bar R)^2W_Q .
\end{align*}
Writing $\hat Q_k-\hat Q=(\hat Q_k-Q_k)+(Q_k-\bar Q)-(\hat Q-\bar Q)$ and using $(x_1+x_2+x_3)^2\le3(x_1^2+x_2^2+x_3^2)$ and $\sum_k\omega_k^2\le\bar\omega$ gives $\EE W_Q\le3(1+\bar\omega)\varsigma_Q^2+3\mathsf h_Q^2$, which tends to zero. Markov's inequality therefore gives $W_Q=o_p(1)$. Moreover, $\hat R-\bar R=(\hat y/\varsigma_y)(\varsigma_y/\hat Q)=O_p(\varsigma_y)$ and $\hat\varsigma_y=O_p\{(\varsigma_y^2+\mathsf h_y^2)^{1/2}\}$. The right-hand side of the last display is therefore $o_p(\varsigma_y^2+\mathsf h_y^2)$. Hence $\hat s^2/(\varsigma_y^2+\mathsf h_y^2)\to1$ in probability, which is the conclusion of (b) for the linearized estimates with $\hat s$ in place of $\hat\varsigma_y$.
On the event $\hat Q\ne0$ we have $|\hat R-\bar R|=|\hat y|/|\hat Q|$ and $\hat\varsigma_R=\hat s/|\hat Q|$. The ratio interval therefore covers $\bar R$ exactly when $|\hat y-\bar y|\le z_{1-a/2}\hat s$. Part 3 uses only the conclusions of (a) and (b). Applied to the linearized estimates with $\hat s$ in place of $\hat\varsigma$, it shows that the probability of this event has limit inferior at least $1-a$ and converges to $1-a$ when $\mathsf h_y^2/\varsigma_y^2\to0$. Since the event $\hat Q\ne0$ has probability tending to one, the same holds for the ratio interval, which proves (d). This completes the proof.
\end{proof}

\subsection{Coverage of the intervals in simulation}\label{app:coverage}

The simulation applies the within-stratum estimators of \Cref{app:reports,app:finite} and the $t$ interval used in the analyses on simulated strata with true values from $10^6$ Monte Carlo draws. Strata have $6$ varying positions with $p=192$ coordinates, a smooth nonlinear response with an analytic gradient and $10$ directions. Each model's exact task stratum-size vectors (strata of at least three prompts, weights $\omega_k\propto n_k-1$) are reproduced under the Gaussian law, a multivariate $t_5$ law and a sparse JumpReLU code $H=\sum_jz_jd_j+e$, the law closest to SAE activations, with $1{,}000$ replicates each.

\begin{table}[h]
\centering\footnotesize
\setlength{\tabcolsep}{3.8pt}
\begin{tabular}{l*{12}{c}}
\toprule
& \multicolumn{4}{c}{Gaussian} & \multicolumn{4}{c}{$t_5$} & \multicolumn{4}{c}{JumpReLU code} \\
\cmidrule(lr){2-5}\cmidrule(lr){6-9}\cmidrule(lr){10-13}
& GPT-2 & 2B & 9B & 8B & GPT-2 & 2B & 9B & 8B & GPT-2 & 2B & 9B & 8B \\
Tasks & 8 & 10 & 10 & 7 & 8 & 10 & 10 & 7 & 8 & 10 & 10 & 7 \\
\midrule
Stein, mean & 0.961 & 0.971 & 0.970 & 0.967 & 0.918 & 0.925 & 0.927 & 0.925 & 0.835 & 0.821 & 0.824 & 0.820 \\
Stein, test-weighted & 0.955 & 0.971 & 0.970 & 0.964 & 0.915 & 0.929 & 0.925 & 0.921 & 0.863 & 0.814 & 0.829 & 0.830 \\
Stein, $K_s\le9$ & 0.974 & 0.981 & 0.978 & 0.978 & 0.925 & 0.934 & 0.939 & 0.936 & 0.804 & 0.792 & 0.796 & 0.777 \\
Stein, $K_s\ge30$ & 0.946 & 0.952 & 0.960 & 0.950 & 0.908 & 0.914 & 0.915 & 0.914 & 0.882 & 0.886 & 0.889 & 0.880 \\
\midrule
$L_j$ & 0.936 & 0.933 & 0.932 & 0.934 & 0.932 & 0.928 & 0.928 & 0.933 & 0.941 & 0.938 & 0.936 & 0.940 \\
$M_j$ & 0.936 & 0.935 & 0.935 & 0.937 & 0.934 & 0.932 & 0.929 & 0.934 & 0.940 & 0.941 & 0.940 & 0.942 \\
$C_j$ & 0.917 & 0.911 & 0.911 & 0.912 & 0.893 & 0.893 & 0.889 & 0.892 & 0.926 & 0.920 & 0.924 & 0.919 \\
\bottomrule
\end{tabular}
\caption{Coverage of nominal $95\%$ stratum-level $t$ intervals at the empirical stratum sizes, averaged over the reproduced tasks (Tasks). 2B, 9B and 8B are Gemma-2-2B, Gemma-2-9B and Llama-3.1-8B. Test-weighted weights tasks by their number of tests. The $K_s$ rows keep tasks with at most nine or at least thirty strata.}\label{tab:sim-emp}
\end{table}

In \Cref{tab:sim-emp} the Stein intervals cover $0.961$ to $0.971$ under the Gaussian law and $0.918$ to $0.927$ under $t_5$. Under the JumpReLU code they cover $0.820$ to $0.835$ on average over tasks, $0.777$ to $0.804$ for tasks with at most nine strata and $0.880$ to $0.889$ for tasks with thirty or more. Separated sparse codes are far from Gaussian and give heavy-tailed within-stratum Stein estimates. The simulation thus shows substantial finite-sample undercoverage of the Stein intervals under the sparse JumpReLU design. Coverage is higher in the groups with more strata, consistent with the many-strata limit of \Cref{prop:cells}. The ablation intervals are much less affected. Those for $L_j$ and $M_j$ cover $0.928$ to $0.942$ under each of the three laws. Those for $C_j$ cover $0.911$ to $0.917$ under the Gaussian law, $0.889$ to $0.893$ under $t_5$ and $0.919$ to $0.926$ under the JumpReLU code. Coverage at these stratum sizes depends on unbiased stratum estimates, which the within-stratum U-statistics provide for strata of at least three prompts. In a design with $60$ strata of fixed size $n_k\in\{3,4,5\}$, the stratum-level intervals around within-stratum U-statistics cover at $0.964$, $0.971$ and $0.957$, whereas intervals around plug-in stratum estimates cover at $0.353$, $0.482$ and $0.555$.

%% file: sections/app_empirics.tex
\section{Experimental setup}\label{app:empirics}

This appendix describes the tasks, models, interventions and estimation shared by the analyses of \Cref{sec:results}. Every empirical result is given model by model for GPT-2, Gemma-2-2B, Gemma-2-9B and Llama-3.1-8B, in this order, with one SAE at one layer of each model. Unless stated otherwise, ``attribution'' is $\alpha(\Delta)$ of \eqref{eq:effect} for the intervention named and every quantity is computed on the evaluation split. Bracketed intervals are $95\%$ percentile bootstrap intervals whose resampling unit is given with each table. Estimates from fewer than thirty direction-task pairs are marked with an asterisk$^\ast$.

\subsection{Data, models and protocols}\label{app:protocols}

\begin{table}[ht]
\centering\scriptsize
\setlength{\tabcolsep}{2.5pt}
\begin{adjustbox}{max width=\textwidth}
\begin{tabular}{@{}>{\raggedright\arraybackslash}p{2.35cm}c>{\raggedright\arraybackslash}p{4.35cm}>{\raggedright\arraybackslash}p{4.2cm}>{\raggedright\arraybackslash}p{1.4cm}rrc@{}}
\toprule
& & & & & \multicolumn{2}{c}{Prompts} & Strata \\
\cmidrule(lr){6-7}
Task & Cat. & Construction & Representative template & Answer & Sel. & Eval. & $n_k\ge3$ \\
\midrule
\texttt{antonym} & L & $62$ pairs in both directions, $4$ tpl. & \texttt{The opposite of "\{x\}" is} & antonym & 317 & 74 & 0\,/\,0\,/\,0 \\
\texttt{country\_capital} & F & $51$ country-capital pairs, $20$ tpl. & \texttt{The capital of \{x\} is} & capital city & 832 & 188 & 2\,/\,3\,/\,3 \\
\texttt{country\_capital\_\allowbreak mcqa} & F & $58$ countries, two prompts each with the capital and three random distractor capitals in random order & \texttt{What is the capital of \{x\}? (A) \{o1\} (B) \{o2\} (C) \{o3\} (D) \{o4\}. Answer:} & option letter & 71 & 15 & 1\,/\,3\,/\,1 \\
\texttt{country\_language} & F & $63$ countries with one common language, $4$ tpl. & \texttt{In \{x\}, people commonly speak} & language & 157 & 45 & 2\,/\,3\,/\,2 \\
\texttt{few\_shot\_\allowbreak induction} & I & $400$ prompts, two demonstration pairs and a query pair from disjoint halves of the antonym list & \texttt{\{a\}:\{b\}, \{c\}:\{d\}, \{x\}:} & antonym of the query word & 264 & 56 & 6\,/\,9\,/\,6 \\
\texttt{ioi} & I & $7$ tpl.\ in two name orders, single-token GPT-2 names, $13$ places, $14$ objects & \texttt{Then, \{x\} and \{y\} went to the \{place\}. \{y\} gave a \{object\} to} & name of the indirect object & 1{,}605 & 391 & 71\,/\,71\,/\,71 \\
\texttt{last\_letter} & L & $132$ common words, $4$ tpl. & \texttt{The last letter of "\{x\}" is} & letter & 334 & 89 & 11\,/\,10\,/\,11 \\
\texttt{person\_birthplace} & F & $50$ person-city pairs, $18$ tpl. & \texttt{\{x\} was born in} & city & 734 & 166 & 4\,/\,4\,/\,5 \\
\texttt{person\_occupation} & F & $40$ person-occupation pairs, $25$ tpl. & \texttt{\{x\} worked as a} & occupation & 784 & 216 & 32\,/\,37\,/\,29 \\
\texttt{simple\_addition} & A & $150$ pairs, $x,y\ge10$, $x+y\le99$, half with a carry, $29$ tpl. & \texttt{\{x\} + \{y\} =} & integer & 1{,}581 & 419 & 50\,/\,46\,/\,43 \\
\texttt{two\_digit\_\allowbreak subtraction} & A & $600$ pairs, $10\le y\le x\le99$, half with a borrow, $1$ tpl. & \texttt{\{x\} - \{y\} =} & integer & 384 & 96 & 9\,/\,9\,/\,9 \\
\midrule
Total & & & & & 7{,}063 & 1{,}755 & 188\,/\,195\,/\,180 \\
\bottomrule
\end{tabular}
\end{adjustbox}
\caption{The eleven tasks, with templates abbreviated tpl.\ and the category (Cat.) factual recall (F), lexical (L), in-context (I) or arithmetic (A). Entity lists are curated and operands are drawn at random. In a template \texttt{\{x\}} marks the entity slot, \texttt{\{a\}:\{b\}} and \texttt{\{c\}:\{d\}} the demonstration pairs and \texttt{\{y\}} the second operand or, in \texttt{ioi}, the subject. The corpus leaves out the paraphrase templates of \texttt{country\_capital}, \texttt{person\_birthplace}, \texttt{person\_occupation} and \texttt{simple\_addition}, the IOI test prompts and a held-out fifth (at least $30$ prompts) of each of the other six tasks. It caps each task at $2{,}000$ prompts. The last column counts the evaluation strata of at least three prompts for GPT-2, the two Gemma models and Llama-3.1-8B. Strata depend on the model through the tokenized prompt length. The primary analyses require strata of at least three prompts. The ablation sensitivity analysis also admits two-prompt strata.}\label{tab:tasks}
\end{table}

\paragraph{Tasks and splits.}
The corpus of \Cref{tab:tasks} holds $8{,}828$ prompts, split at random into a selection split of $7{,}063$ prompts ($80\%$) and an evaluation split of $1{,}765$. Ten evaluation prompts that also occur in the selection split are removed from every evaluation analysis, which leaves $1{,}755$. A stratum fixes a task, an answer and a prompt length.

\paragraph{Models, SAEs and outcome.}
GPT-2 \citep{radford2019language} is read with the residual-stream SAE at the input of block $6$ ($24{,}576$ latents, \citealp{bloom2024gpt2residualsaes}). Gemma-2-2B and Gemma-2-9B \citep{gemmateam2024gemma2} are read with the Gemma Scope JumpReLU residual-stream SAEs \citep{lieberum2024llamascope} at layer $12$ (width $16{,}384$, average $L_0$ $41$) and layer $22$ (width $131{,}072$, average $L_0$ $32$). Llama-3.1-8B \citep{dubey2024llama3} is read with the Llama Scope residual-stream SAE at layer $16$ ($32{,}768$ latents, \citealp{he2024llamascope}). All four decoders have full rank $d$ ($768$, $2{,}304$, $3{,}584$ and $4{,}096$). Each model is used under its published licence (MIT for GPT-2, the Gemma Terms of Use for Gemma~2, the Llama 3.1 Community License for Llama~3.1) and each SAE under the licence of its public release. The outcome of every probe, gradient and intervention is the teacher-forced log-probability of the full answer, $\ell=\sum_k\log p(y^\ast_k\mid\text{prompt},y^\ast_{<k})$, standardized within task with constants fixed on the selection split. Because the standardization is affine, $F$ denotes $\psi\circ F$ (\Cref{asm:cut}).

\paragraph{Gradients and compute.}
One backward pass per evaluation prompt returns $\nabla_H\ell$ at every prompt and answer position, in single precision. Analytic directional derivatives agree with central finite differences along $100$ latent and $20$ gradient directions on every model (for Llama-3.1-8B in double precision). The computations use NVIDIA A100 80GB GPUs on an institutional cluster and require at most about two GPU-hours per model and analysis.

\paragraph{Interventions.}
Every intervention is additive at the SAE hook, keeps the SAE error term and leaves the first position unchanged (\Cref{asm:cut}). Zero-ablation sets the code of one latent to zero at every position but the first. Mean-ablation sets it to the latent's task mean on the selection split. The analyses ablate one latent at a time. The worked example also ablates sets of latents (\Cref{app:case}). Steering adds $\pm\varepsilon\,d_j$ at the final prompt position, a fixed-size edit whose edit-gradient covariance in \eqref{eq:gap} is zero. With the unit edit $u=E_{\rm fin}^\top d_j$, the realized steering effects give the central difference quotient $T\coloneqq\{F(H+\varepsilon u)-F(H-\varepsilon u)\}/(2\varepsilon)$ at $\varepsilon=0.25\,\sigma_{tj}$, pooled over strata with the weights of $\alpha$, where $\sigma_{tj}$ is the pooled within-stratum standard deviation of the latent's final-position score on the evaluation split. On the scale of $u$, $T$ is compared with the attribution $d_j^\top E_{\rm fin}\gamma=\alpha(\Delta)/\varepsilon$ (mean gradient $\gamma=\EE\nabla F(H)$) and with the probe-implied change $(Au)\beta=\beta$ (final-position score $A=d_j^\top E_{\rm fin}$). Because $\sigma_{tj}$ is estimated on the evaluation split, the steering edit is not fixed by the selection split. The steering comparisons of \Cref{app:x6} are therefore descriptive, lie outside the inference of \Cref{thm:inference} and carry descriptive bootstrap intervals.

\paragraph{Within-stratum estimation.}
Stratum-specific quantities are estimated without bias in strata with at least three prompts and pooled within task with weights $\omega_k\propto n_k-1$ (\Cref{prop:cells}). A ratio such as a probe coefficient is the ratio of its pooled numerator and its pooled denominator (\Cref{app:e1r}). Pooling prompts over strata before estimation would change the estimand, since the coefficient of a probe pooled over all strata of a task splits exactly into a within-stratum and a between-stratum part. Over the $50$, $244$, $481$ and $24$ selected direction-task pairs the median absolute share of the between-stratum part is $0.632$, $0.816$, $0.818$ and $0.662$. A pooled probe is thus mostly a contrast between answers and prompt lengths, which is why every estimate is computed within a stratum.

\subsection{Prespecified and post hoc analyses}\label{app:prespec}

The fixed-direction analysis and its variants (\Cref{app:e1r}), the coefficient-selected steering and its second-dictionary replication (\Cref{app:x6,app:x7}), the ablation decomposition (\Cref{app:e8}) and the single-latent ablations (\Cref{app:singletons}) are prespecified. A prespecified analysis has its estimand and decision rule fixed before its outcome is computed. Three further prespecified criteria of these plans concern set ablations. The first requires the prompt-specific gradient prediction of the finite ablation of each task's selection to be more accurate than the probe-based one. It is met in all four models under zero- and mean-ablation, with paired intervals excluding zero. Each of the other two must hold in at least three models. The second requires the first-order regime at reduced edit scales of the set ablation, with a near-quadratic remainder of slope in $[1.5,2.5]$. It is met in all four models. The third requires interactions to dominate the remainder of the full set ablation, with an interaction share above $0.5$. It is met only on Gemma-2-9B. The remaining analyses are post hoc and designed after the aggregate results. They comprise the larger steering steps and the maximum-activation displacements (\Cref{app:x6}), the sampling-noise analysis (\Cref{app:noise}), the probe-to-ablation breakdown and the positional analysis (\Cref{app:x8,app:position}), the probes on FDFI's averaged summaries (\Cref{app:fdfi-summary}), the observed span of a stratum with its prespecified readout (\Cref{app:span-check}), the rankings against steering responses and the held-out set ablations (\Cref{app:choice}) and the worked example (\Cref{app:case}). The analysis specifications are released with the code.

\subsection{The FDFI selection step}\label{app:selection}

FDFI produces the candidates analysed in \Cref{sec:results}. We use its Gaussian transport map with squared loss, under which FDFI coincides with disentangled feature importance with the Bures-Wasserstein map \citep{du2025disentangled}. The flow-matching map of \citet{chen2026fdfi} is not used. On the selection split it fits one ridge predictor $f$ (penalty $\nu$) of the within-task standardized outcome, linear in the active latent activations $z\in\RR^{K_{\rm act}}$ averaged over a set $\mathcal P$ of positions, with a task-specific intercept. The active set holds the latents whose selection-split firing rates lie in an interval chosen without reference to the target ($K_{\rm act}=679$, $329$, $648$ and $396$). Within a task, with $\bar z$ and $\Sigma_z\succ0$ the mean and covariance of $z$, FDFI uses the Gaussian transport map $z=\bar z+\Sigma_z^{1/2}\xi$ with the symmetric square root, under which $\xi\coloneqq\Sigma_z^{-1/2}(z-\bar z)$ has $\EE\xi=0$ and $\Cov(\xi)=I$ for any law of $z$. The implementation combines the decorrelated-coordinate importance estimates using squared transport-Jacobian entries, giving $\hat\phi_l=\sum_k(\hat L_{lk})^2\hat\phi_k^Z$, where $\hat L$ is the fitted Gaussian transport matrix. Here $\phi_l$ denotes the population importance targeted by $\hat\phi_l$ under the specified Gaussian-transport procedure. The selection $S_t$ applies Benjamini-Hochberg at $q=0.05$ within task $t$ \citep{benjamini1995controlling} to the one-sided statistics $\hat\phi_l/\widehat{\rm se}_l$ for $H_0:\phi_l=0$ against $\phi_l>0$. The standard errors are those of the package's confidence intervals with the gap margin and the null-threshold adjustment. The $p$-values are one-sided normal-reference $p$-values against zero computed from these standard errors. The selections contain $50$, $244$, $481$ and $24$ direction-task pairs with $41$, $140$, $287$ and $20$ distinct latents. With fitting choices and algorithmic randomness held fixed, the estimated scores and selections depend only on the observed activations, targets and stratum labels. They therefore remain unchanged under the response perturbations of \Cref{prop:discrete}, applied stratum by stratum. A statement about intervening on a selected latent therefore rests on the attribution rather than on its FDFI statistic.

\subsection{Mean edits and the observed span of a stratum}\label{app:span-check}

This section concerns the empirical law of the observed prompts of each evaluation stratum, for which $V$ is the span of the observed activation differences. It assesses the geometric condition of \Cref{prop:discrete} in the evaluation design and makes no statement about the dimension of the population support. \Cref{prop:discrete} leaves the empirical-average attribution unrestricted whenever the empirical mean edit lies outside $V$. For every FDFI-selected pair and every evaluation stratum of at least three prompts we compute the share $\|\bar u_{V^\perp}\|/\|\bar u\|$ of the mean edit outside $V$, treating positions whose activations are numerically constant within the stratum (relative tolerance $10^{-6}$) as constant. The full-space shares agree with these at the displayed precision. Each pair value averages its strata with weights $n_k-1$. Under unit steering at the final position the median share over pairs is $99.94$, $99.97$, $99.96$ and $99.94\%$. The smallest pair value is $0.973$ and no stratum has a share below $10^{-6}$. Under zero-ablation every nonzero mean edit lies outside $V$, with median shares of $0.994$ to $0.998$. In the remaining strata the edit vanishes at every position that varies within the stratum. At this numerical tolerance every assessed stratum has rank $n_k-1$ and hence affinely independent activations, except one IOI stratum of six prompts in each model. The secant reading of \Cref{lem:finite-stein}, which requires affinely independent support of the stratum's empirical law, therefore applies to every other stratum.

%% file: sections/app_reports.tex
\section{The probe against the intervention attribution}\label{app:reports}

This appendix supports the statements of \Cref{sec:results} that the probe-implied change departs from the intervention attribution in size and in sign. \Cref{app:prespec} states which of its analyses were prespecified. The outcome is the within-task standardized log-probability of the answer, $d_j$ is the unit decoder direction of latent $j$ and quantities are pooled over strata of at least three prompts with weights $\omega_k\propto n_k-1$. Unless stated otherwise the probe is fitted on the final-position decoder score $A=d_j^\top E_{\rm fin}$, with coefficient $\beta$. Pairs are direction-task pairs.

\subsection{The pipeline's own directions under a fixed edit}\label{app:e1r}

Under a unit steering intervention the attribution of an FDFI-selected direction is a few percent of its probe coefficient, whereas the Stein residual and the omitted-variable term are of the order of the coefficient.

\paragraph{Design.} The analysis covers the selected pairs $(t,j)$, $j\in S_t$, that have an estimate (\Cref{tab:e1r}) under the unit steering edit $\Delta\equiv u=E_{\rm fin}^\top d_j$. Since $Au=1$, the probe-implied change is $\rho=\beta$, the edit-gradient covariance vanishes and \eqref{eq:gap} reads $\beta-\alpha=\st_{u'}+\cm$, with $\alpha=d_j^\top\gamma_{\rm fin}$ for the mean gradient $\gamma_{\rm fin}$ at the final position. Within each stratum the covariance and variance of the score, $\alpha$, $\st$ and $\cm$ are estimated by unbiased U-statistics. The pooled coefficient $\beta=\sum_kq_k\beta_k$, $\alpha$, $\st$ and $\cm$ average their stratum-specific values with weights $q_k\propto\omega_k\sigma_k^2$, where $\sigma_k^2$ is the score variance of stratum $k$ (\Cref{app:proofs:gap}). Intervals resample strata within task ($2{,}000$ replicates).

\paragraph{Results.} The median $|\alpha/\beta|$ is $0.028$ to $0.046$ and the ratio is below $0.1$ for $74$ to $83\%$ of pairs, with medians of $0.020$ to $0.057$ across the prespecified variants (\Cref{tab:e1r}). The signs of $\alpha$ and $\beta$ disagree for $22$ to $36\%$ of pairs and for $15$ to $43\%$ across the variants, shares that are descriptive because sampling noise alone would produce more sign disagreement (\Cref{app:noise}). The median $|\st/\beta|$ is $0.614$, $0.922$, $0.728$ and $0.583$ and the median $|\cm/\beta|$ is $0.514$, $0.593$, $0.463$ and $0.406$. For a law with $\Sigma_H\succ0$ the omitted-variable term of the unit steering edit vanishes when $\mathrm{VIF}_j=1$ (\Cref{lem:vif}), a lemma that does not apply to the singular $\Sigma_H$ of the empirical stratum laws. A proxy with Ledoit-Wolf shrinkage restricts the pooled within-stratum covariance of the final-position activation on the selection split to the span of the active decoder directions. Restriction lowers or preserves the variance inflation, which makes the proxy an approximate lower bound on the pooled $\mathrm{VIF}_j$. Over all $50$, $244$, $481$ and $24$ selected directions its median is $141.647$, $78.035$, $122.873$ and $91.910$ and it exceeds $10$ for $98$, $97$, $99$ and $96\%$ of directions.

\begin{table}[h]
\centering\scriptsize
\setlength{\tabcolsep}{3pt}
\begin{adjustbox}{max width=\textwidth}
\begin{tabular}{llcccccccccc}
\toprule
& & & \multicolumn{2}{c}{Primary, $n_k\ge3$} & \multicolumn{2}{c}{Trim $|\beta|$} & & & \multicolumn{2}{c}{Larger strata} & \\
\cmidrule(lr){4-5}\cmidrule(lr){6-7}\cmidrule(lr){10-11}
Model & Statistic & $N$ (tasks) & estimate [CI] & $<0.1$ & $10\%$ & $20\%$ & Eq.\ tasks & LOTO & $n_k\ge4$ & $n_k\ge6$ & Direct \\
\midrule
GPT-2 & $|\alpha/\beta|$ & 46 (8) & 0.046 [0.030, 0.059] & 0.739 & 0.026 & 0.024 & 0.054 & 0.031 to 0.057 & 0.036 & 0.035 & 0.049 \\
 & sign disagr. & & 0.348 [0.283, 0.435] & & 0.317 & 0.306 & 0.365 & 0.318 to 0.378 & 0.333 & 0.375 & 0.391 \\
\addlinespace
Gemma-2-2B & $|\alpha/\beta|$ & 230 (10) & 0.032 [0.022, 0.036] & 0.770 & 0.027 & 0.024 & 0.035 & 0.031 to 0.034 & 0.034 & 0.044 & 0.030 \\
 & sign disagr. & & 0.352 [0.317, 0.426] & & 0.353 & 0.337 & 0.321 & 0.343 to 0.365 & 0.391 & 0.430 & 0.400 \\
\addlinespace
Gemma-2-9B & $|\alpha/\beta|$ & 457 (10) & 0.028 [0.019, 0.035] & 0.751 & 0.023 & 0.020 & 0.022 & 0.023 to 0.032 & 0.023 & 0.027 & 0.031 \\
 & sign disagr. & & 0.363 [0.341, 0.422] & & 0.355 & 0.353 & 0.360 & 0.350 to 0.379 & 0.384 & 0.415 & 0.396 \\
\addlinespace
Llama-3.1-8B & $|\alpha/\beta|$ & 23 (7)$^\ast$ & 0.032 [0.022, 0.051] & 0.826 & 0.030 & 0.027 & 0.032 & 0.029 to 0.038 & 0.031 & 0.056 & 0.027 \\
 & sign disagr. & & 0.217 [0.174, 0.348] & & 0.150 & 0.167 & 0.287 & 0.182 to 0.250 & 0.300 & 0.200 & 0.217 \\
\bottomrule
\end{tabular}
\end{adjustbox}
\caption{The fixed-direction analysis and its prespecified variants, with the median $|\alpha/\beta|$ and the share of pairs with $\operatorname{sign}\alpha\ne\operatorname{sign}\beta$. Column $<0.1$ is the share with $|\alpha/\beta|<0.1$. Trim drops the smallest $10\%$ or $20\%$ of $|\beta|$, Eq.\ tasks weights tasks equally and LOTO is the leave-one-task-out range. Larger strata leave $42$, $202$, $417$ and $20$ pairs at $n_k\ge4$ and $40$, $107$, $277$ and $15$ at $n_k\ge6$. Direct replaces $\alpha$ by the unweighted mean of $d_j^\top g_{\rm fin}$ over the task's evaluation prompts.}\label{tab:e1r}
\end{table}

\paragraph{Dependence on the summary.}\label{app:span}
The coefficient that a probe assigns to a latent depends on the summary that contains its score. A second probe for each selected pair reads a summary $S=AH$ of $6$ to $18$ rows, the score $d_j^\top h$ at the final position, at up to eight preceding prompt positions and at up to four answer positions common to the task's strata (on GPT-2 also eight within-task principal components of the mean activation at the non-final positions). Since $A\Sigma_HA^\top$ is singular in every stratum with at most as many prompts as rows, its coefficient $\beta_S$ uses within-stratum moments of $(S,Y)$ pooled over the task's strata, to which the pooled identity of \Cref{app:proofs:gap} applies whenever the pooled summary covariance is positive definite. Tasks with pooled within-stratum degrees of freedom of at least three times the number of rows keep $29$, $92$, $225$ and $15$ pairs. Since $Au=e_1$, the probe-implied change for unit steering is $\beta_S[1]$, whose sign is opposite to that of $\beta$ for $108$ of the $361$ pairs ($6/29$, $33/92$, $65/225$, $4/15$), whereas $\alpha$ is the same for both. Changing $A$ changes $\rho$ and moves $\st+\cm$ by exactly the difference, as \eqref{eq:gap} requires for every linear probe.

\subsection{Steering of coefficient-selected directions}\label{app:x6}

The attribution predicts the steering effect of the directions with the largest coefficients, which exceed that effect by two orders of magnitude. Because \Cref{thm:gap} holds for every selection rule, selecting by the coefficient tests its intervention predictions directly.

\paragraph{Design.} For each task with at least two evaluation strata of three or more prompts, the ten active latents whose decoder scores have the largest within-stratum coefficients $|\beta^{\rm sel}_{tj}|$ on the selection split form $\mathcal S^{\beta}_t$. No FDFI quantity and no evaluation-split data enter the selection. On the evaluation split $\beta$ is re-estimated and $T_{tj}$ is the central difference quotient of steering at $\pm0.25\,\sigma_{tj}$, pooled with the weights of $\alpha$ (\Cref{app:protocols}). Because the score's pooled within-stratum standard deviation $\sigma_{tj}$ is estimated on this split, the comparisons are descriptive and lie outside \Cref{thm:inference}. A standardized change $\delta$ multiplies the answer probability by $\exp(\delta s_t)$, with $s_t$ the within-task standard deviation of the answer's log-probability on the selection split, which puts the steering intervention at a median change of $0.02$ to $0.09\%$ in the answer probability. The prespecified primary contrast is the four-model average of $\mathrm{Gain}_m=\mathrm{mean}_t\,\mathrm{mean}_{j\in \mathcal S^{\beta}_t}[\,|\sigma_{tj}(\beta_{tj}-T_{tj})|-|\sigma_{tj}(\alpha_{tj}-T_{tj})|\,]$. Its decision rule requires an interval excluding zero and $\mathrm{Gain}_m>0$ in every model. The descriptive interval comes from a joint cross-model bootstrap that draws tasks once per replicate for all models and then resamples strata ($2{,}000$ replicates).

\paragraph{Results.} The prespecified criterion holds (\Cref{tab:x6}), with $\mathrm{Gain}=0.345$, descriptive interval $[0.267,0.487]$ and a positive contrast in every model. The attribution matches $T$ in sign and is closer to $T$ than $\beta$ is for $380$ of $380$ pairs, with median $|T/\alpha|$ of $1.000$ to $1.001$. A no-intercept calibration of $T$ on $\alpha$ has slope $0.997$ to $1.000$ and $R^2>0.999$, against $R^2\le0.280$ for $\beta$. The coefficient is about $130$ to $400$ times the steering effect in median magnitude and has the opposite sign for $50$, $43$, $21$ and $40\%$ of pairs. The selection-split coefficient $\beta^{\rm sel}$ gives a larger contrast still ($0.344$, $0.372$, $0.389$ and $0.449$). Selecting directions to maximize $|\beta|$ leaves the Stein residual and the omitted-variable term of the order of $\beta$ (median $|\st/\beta|$ $0.404$ to $0.730$, $|\cm/\beta|$ $0.506$ to $0.647$). These directions lie almost entirely outside the FDFI selections (task-mean share $0$ to $0.030$), which makes the result a statement about the coefficient as a selection criterion. In a post hoc descriptive extension, steering at $2$, $5$ and $10$ within-stratum standard deviations leaves the attribution with the sign of $T$ for $753$, $739$ and $712$ of the $756$ FDFI-selected pairs and for $380$, $380$ and $379$ of the $380$ coefficient-selected pairs. No step reaches the prespecified breakdown threshold, at which $|T-\alpha|>\frac12|\alpha|$ for a majority of pairs in three of four models.

\paragraph{Displacements scaled by the maximum activation.} The same extension adds displacements of $1$, $2$ and $5$ times the evaluation-split maximum activation of each latent, its largest code over evaluation prompts and positions after the first, converted to a displacement with the decoder norm. These are added displacements, not settings of the activation to its maximum. The prespecified breakdown threshold concerns the standard-deviation grid and is not applied here. The results are descriptive. The first-answer-token flip rate averages, for each pair, the indicator that the top-ranked first answer token changes over both steering signs and the pair's prompts. Pairs are then weighted equally. Closeness compares $\alpha$ and $\beta$ with the central difference quotient and does not assess each one-sided intervention effect. Breakdown is the fraction of pairs with $|T-\alpha|>\frac12|\alpha|$. At one maximum activation the flip rate is $0.219$, $0.080$, $0.048$ and $0.101$ for the FDFI-selected pairs ($46$, $230$, $457$ and $23$) and $0.071$, $0.030$, $0.018$ and $0.027$ for the coefficient-selected pairs ($90$, $100$, $97$ and $89$). The attribution is closer to $T$ than $\beta$ for $0.978$, $0.996$, $0.987$ and $1.000$ of the FDFI-selected pairs and for every coefficient-selected pair, with median $|T/\alpha|$ of $0.965$ to $1.006$. Its sign agreement with $T$ is $0.947$ to $0.978$ for the FDFI-selected pairs and $0.930$ to $0.978$ for the coefficient-selected pairs, against $0.646$ to $0.826$ and $0.500$ to $0.763$ for $\beta$. Breakdown occurs for $0.043$ to $0.239$ and $0.072$ to $0.150$ of these pairs. At five maximum activations the flip rate rises to $0.295$ to $0.678$ and $0.152$ to $0.339$ and breakdown to $0.652$ to $0.739$ and $0.389$ to $0.600$, while the attribution stays closer to $T$ than $\beta$ for $0.950$ to $1.000$ of pairs.

\begin{table}[h]
\centering\scriptsize
\setlength{\tabcolsep}{2.5pt}
\begin{adjustbox}{max width=\textwidth}
\begin{tabular}{lcccccccccc}
\toprule
& & & \multicolumn{2}{c}{$\alpha$ vs.\ $T$} & & & & & & \\
\cmidrule(lr){4-5}
Model & Pairs (tasks) & $\mathrm{Gain}_m$ [CI] & sign agr. & $|T/\alpha|$ & $|T/\beta|$ & opp.\ sign $\alpha,\beta$ & $|\st/\beta|$ & $|\cm/\beta|$ & $R^2$ ($\alpha$ / $\beta$) & in $S_t$ \\
\midrule
GPT-2 & 90 (9) & 0.333 [0.248, 0.456] & 1.000 & 1.000 & $7.60\times10^{-3}$ & 0.500 & 0.404 & 0.647 & 1.000 / $-$0.022 & 0 \\
Gemma-2-2B & 100 (10) & 0.311 [0.219, 0.491] & 1.000 & 1.000 & $3.68\times10^{-3}$ & 0.430 & 0.730 & 0.599 & 1.000 / $6.70\times10^{-3}$ & 0.010 \\
Gemma-2-9B & 100 (10) & 0.346 [0.182, 0.699] & 1.000 & 1.001 & $2.50\times10^{-3}$ & 0.210 & 0.637 & 0.613 & 1.000 / 0.280 & 0.030 \\
Llama-3.1-8B & 90 (9) & 0.390 [0.270, 0.517] & 1.000 & 1.000 & $3.36\times10^{-3}$ & 0.400 & 0.609 & 0.506 & 1.000 / 0.030 & 0 \\
\midrule
Aggregate & 380 (38) & 0.345 [0.267, 0.487] & 380/380 & & & & & & & \\
\bottomrule
\end{tabular}
\end{adjustbox}
\caption{Steering the directions with the largest selection-split $|\beta|$ at $\pm0.25$ within-stratum standard deviations. The contrast $\mathrm{Gain}_m$ is positive when $\alpha$ predicts $T$ better than $\beta$ and its intervals are descriptive. Ratios are medians over pairs. The $R^2$ of the no-intercept calibration of $T$ on the predictor uses the centred total sum of squares and is negative for a fit worse than the mean. Column in $S_t$ is the task-mean share of directions in the FDFI selection.}\label{tab:x6}
\end{table}

\paragraph{A second dictionary.}\label{app:x7}
A gap that is a property of the probe rather than of one dictionary reappears under a second SAE. A prespecified rule picks for each model an alternate SAE at the same hook with a different width and the nearest average $L_0$, eligible if its fraction of variance unexplained (FVU) is at most $0.5$ and at most twice the primary's. The GPT-2 alternate has FVU $0.729$ against a threshold of $0.278$ and is excluded. On both dictionaries the selection rule of this subsection acts on the latents firing at the final position on at least $10\%$ of selection-split prompts. The prespecified criteria are (i) a task-balanced median $|\alpha/\beta|\le0.10$, (ii) a sign disagreement $\ge0.15$ and (iii) $|C|>|M|$ for $\ge10\%$ of directions. Both dictionaries meet (i) and (ii) on the three eligible models (\Cref{tab:x7}), with median $|\alpha/\beta|\le9.51\times10^{-3}$ and sign disagreement $0.300$ to $0.420$. Criterion (iii) holds on Llama-3.1-8B only and fails on the Gemma models under both dictionaries ($0.030$ to $0.040$). Under the FDFI selections of the same Gemma models $|C|>|M|$ holds for $34$ and $39\%$ of latents (\Cref{tab:e8a}), which points to the selection rule rather than the dictionary as the source of the low rate.

\begin{table}[h]
\centering\scriptsize
\setlength{\tabcolsep}{3pt}
\resizebox{\textwidth}{!}{\begin{tabular}{llcccc}
\toprule
Model & Alternate SAE (primary) & FVU alt.\ / prim.\ (thr.) & Pool prim.\ / alt. & Primary (i) / (ii) / (iii) & Alternate (i) / (ii) / (iii) \\
\midrule
GPT-2 & \texttt{res\_scefr-ajt} (\texttt{res-jb}) & 0.729 / 0.139 (0.278) & N/A & not eligible & N/A \\
Gemma-2-2B & 32k, $L_0$ 40 (16k, $L_0$ 41) & 0.084 / 0.092 (0.184) & 86 / 62 & $4.97\times10^{-3}$ / 0.420 / 0.040 & $5.02\times10^{-3}$ / 0.390 / 0.040 \\
Gemma-2-9B & 16k, $L_0$ 35 (131k, $L_0$ 32) & 0.056 / 0.045 (0.090) & 46 / 71 & $9.51\times10^{-3}$ / 0.300 / 0.030 & $5.95\times10^{-3}$ / 0.390 / 0.040 \\
Llama-3.1-8B & \texttt{l16r\_32x} (\texttt{l16r\_8x}) & 0.243 / 0.275 (0.500) & 409 / 278 & $4.75\times10^{-3}$ / 0.389 / 0.111 & $4.07\times10^{-3}$ / 0.356 / 0.111 \\
\bottomrule
\end{tabular}}
\caption{A second dictionary at the same hook, with the candidate pool and the criteria (i) to (iii) of \Cref{app:x7}. The FVU threshold is $\min\{0.5,\,2\,\mathrm{FVU}_{\rm prim}\}$.}\label{tab:x7}
\end{table}

\subsection{Gap and sampling noise in the coefficient}\label{app:noise}

The gap between the coefficient and the attribution exceeds the sampling noise in $\hat\beta$. Noise alone could produce a small $|\alpha/\beta|$ if the coefficient equalled the attribution, through a large $|\hat\beta|$ of random sign.

\paragraph{Design.} The FDFI-selected pairs of \Cref{app:e1r} and the coefficient-selected pairs of \Cref{app:x6} are tested, the latter with $\beta$ estimated on evaluation prompts not used to select them. For the unit steering intervention the gap $G\coloneqq\beta-\alpha=\st_{u'}+\cm$ is a ratio of pooled sums over the denominator $\sum_k\omega_k\widehat{\VV}_k$ shared by $\beta$ and $\alpha$ and carries the linearized stratum-level standard error of \Cref{prop:cells} with the factor $K_s/(K_s-1)$. Each $\hat G/\widehat{\rm se}(\hat G)$ is referred two-sided to $t_{K_s-1}$, with Benjamini-Hochberg (BH) control at $q=0.05$ within model. The ratio $R_m\coloneqq\sum_p|\hat\alpha_p|/\sum_p|\hat\beta_p|$ over a model's pairs is insensitive to small denominators and has a task-resampled interval ($10{,}000$ replicates). The noise-only simulation of the null $\beta_p=\alpha_p$ draws $\beta^\ast_p=\hat\alpha_p+\widehat{\rm se}(\hat\beta_p)W_p$ with independent $W_p\sim t_{K_s-1}$, where $K_s$ is the number of evaluation strata of pair $p$, ($20{,}000$ replicates) and computes the median of $|\hat\alpha_p/\beta^\ast_p|$ and the share of opposite signs. The prespecified decision rule requires in every model a null probability below $0.05$ and a task-resampled upper limit of $R_m$ below $0.10$.

\paragraph{Results.} The criterion holds in every model and in both populations (\Cref{tab:noise}). Per-pair tests reject a zero gap for $227$ of the $756$ FDFI-selected pairs and for $177$ of the $380$ coefficient-selected pairs. The coefficient is the larger in magnitude in all $177$ rejections among the coefficient-selected pairs and in $225$ of the $227$ among the FDFI-selected pairs, with both exceptions on Gemma-2-9B. Under the noise-only simulation the median ratio would lie between $0.067$ and $0.226$ for the FDFI-selected pairs and every one of the $20{,}000$ simulated medians exceeds the observed one in every model and in both populations. Sign disagreement over all pairs is descriptive, since under the noise-only simulation the median share of opposite signs is $0.391$ to $0.453$ among the FDFI-selected pairs and $0.478$ to $0.490$ among the coefficient-selected pairs, against observed shares of $0.217$ to $0.363$ and $0.210$ to $0.500$ (\Cref{tab:e1r,tab:x6}). Evidence of sign disagreement beyond noise therefore rests on the pairs whose coefficient and attribution are both significantly nonzero under two separate BH families at $0.05$ within model. Of these, $27$ of $141$ FDFI-selected pairs and $22$ of $80$ coefficient-selected pairs have opposite signs.

\begin{table}[h]
\centering\scriptsize
\setlength{\tabcolsep}{3pt}
\resizebox{\textwidth}{!}{%
\begin{tabular}{llccccccc}
\toprule
& & & & & \multicolumn{2}{c}{Rejections of $G=0$} & Noise-only null & Both significant, \\
\cmidrule(lr){6-7}
Pairs & Model & $N$ (tasks) & Median $|\alpha/\beta|$ & $\sum|\alpha|/\sum|\beta|$ [CI] & all & $K_s\ge10$ & median [q5] & opposite sign \\
\midrule
FDFI-selected & GPT-2 & 46 (8) & 0.046 & 0.036 [0.019, 0.064] & 35 & 34/38 & 0.226 [0.156] & 5/20 \\
 & Gemma-2-2B & 230 (10) & 0.032 & 0.032 [0.022, 0.048] & 65 & 50/92 & 0.080 [0.066] & 11/42 \\
 & Gemma-2-9B & 457 (10) & 0.028 & 0.038 [0.021, 0.071] & 110 & 99/225 & 0.067 [0.059] & 9/67 \\
 & Llama-3.1-8B & 23 (7)$^\ast$ & 0.032 & 0.033 [0.023, 0.049] & 17 & 13/14 & 0.170 [0.103] & 2/12 \\
\addlinespace
Coefficient-selected & GPT-2 & 90 (9) & $7.61\times10^{-3}$ & $9.35\times10^{-3}$ [$6.07\times10^{-3}$, 0.013] & 43 & 36/40 & 0.030 [0.023] & 6/16 \\
 & Gemma-2-2B & 100 (10) & $3.70\times10^{-3}$ & $6.46\times10^{-3}$ [$3.69\times10^{-3}$, 0.012] & 47 & 36/40 & 0.016 [0.012] & 10/27 \\
 & Gemma-2-9B & 100 (10) & $2.49\times10^{-3}$ & $2.96\times10^{-3}$ [$1.95\times10^{-3}$, $9.26\times10^{-3}$] & 36 & 32/40 & 0.010 [$8.14\times10^{-3}$] & 0/15 \\
 & Llama-3.1-8B & 90 (9) & $3.41\times10^{-3}$ & $5.75\times10^{-3}$ [$2.67\times10^{-3}$, $9.88\times10^{-3}$] & 51 & 38/40 & 0.016 [0.012] & 6/22 \\
\bottomrule
\end{tabular}}
\caption{Per-pair tests of the gap $G=\beta-\alpha$ (BH at $0.05$ within model) and the noise-only null of \Cref{app:noise}. The pooled ratio has a task-resampled interval. Column $K_s\ge10$ counts rejections among pairs estimated from at least ten strata. The null columns give the median and fifth percentile of the simulated median ratio. The last column counts opposite signs among the pairs whose $\alpha$ and $\beta$ are both significant.}\label{tab:noise}
\end{table}

\subsection{The probe against single-latent ablations}\label{app:x8}

Under zero-ablation the probe-implied change is often zero or negligible where the intervention has a non-negligible attribution, almost always because the latent never fires at the final position. The analysis is post hoc and descriptive (\Cref{app:prespec}).

\paragraph{Design.} The analysis uses the single-latent zero- and mean-ablations of \Cref{app:e8} on the pairs of \Cref{tab:x8}. Within stratum $k$ the probe-implied change is $\rho_k=(A\bar u_k)\beta_k$, with $A\bar u_k$ the mean change of the final-position score and $\beta_k$ the stratum's coefficient. With the attribution $\alpha_k=\mathrm{mean}_i\,\Delta_i^\top g_i$, \eqref{eq:gap} reads $\rho_k-\alpha_k=\st_k+\cm_{{\rm fin},k}+\cm_{{\rm oth},k}-\ec_k$ in empirical moments, with the omitted-variable term split into the final position and the positions the probe does not read. Under zero-ablation $\ec_k$ is the term $C_j$ of \eqref{eq:ablate}. At the threshold $\delta_0=10^{-3}$ in standardized units, K0 holds the pairs with $|\alpha|\le\delta_0$. Among the others, K1 has $\rho=0$ because the latent never fires at the final position (zero-ablation), K2 has $|\rho|\le\delta_0$ outside K1 and K3 and K4 have $|\rho|>\delta_0$ with opposite and equal signs. The component share of a term is the mean of its magnitude, equally weighted over tasks, over $|\st|+|\cm_{\rm fin}|+|\cm_{\rm oth}|+|\ec|$. Intervals resample tasks within model ($2{,}000$ replicates).

\paragraph{Results.} Under zero-ablation the share of K1 and K2 is $0.370$, $0.296$ and $0.287$ on GPT-2 and the two Gemma models and $0.043$ on Llama-3.1-8B (\Cref{tab:x8}). Almost all of these pairs are in K1 ($15$ of $17$, $67$ of $68$, $130$ of $131$ and $1$ of $1$), whose latents never fire at the final position, the one position the probe reads. For them $A\bar u_k=0$ in every stratum, which gives $\st=\cm_{\rm fin}=0$ and places the whole attribution $\alpha=-\cm_{\rm oth}+\ec$ at unread positions (\Cref{cor:ablate}\ref{cor:ablate:blind}). A further $11$ to $24\%$ of pairs have opposite signs. Among same-sign pairs the probe overstates the attribution by a median factor of $7.502$ to $10.959$. Mean-ablation, which shifts the final-position score wherever the latent does not fire, leaves K1 empty. The component shares of \Cref{fig:headline}C locate the gap. The largest share falls to the omitted-variable term at unread positions under zero-ablation on GPT-2 and both Gemma models ($37$ to $45\%$), to the omitted-variable term at the final position on Llama-3.1-8B ($47\%$) and to the Stein residual under mean-ablation ($38$ to $51\%$). Over both ablations $\st$ and $\cm$ carry $89$ to $98\%$ of the summed magnitude and the edit-gradient covariance $2$ to $11\%$. The covariance is a large part of the attribution itself ($|C|>|M|$ for $17$ to $52\%$ of latents, \Cref{tab:e8a}) but a small part of a gap dominated by the probe's overstatement.

\begin{table}[h]
\centering\scriptsize
\setlength{\tabcolsep}{3pt}
\begin{adjustbox}{max width=\textwidth}
\begin{tabular}{llcccccccccccc}
\toprule
& & & \multicolumn{5}{c}{Counts of pairs} & & & \multicolumn{4}{c}{Component shares} \\
\cmidrule(lr){4-8}\cmidrule(lr){11-14}
Model & Ablation & $N$ (tasks) & K0 & K1 & K2 & K3 & K4 & K1+K2 [CI] & K3 [CI] & $\st$ & $\cm_{\rm fin}$ & $\cm_{\rm oth}$ & $\ec$ \\
\midrule
GPT-2 & zero & 46 (8) & 5 & 15 & 2 & 11 & 13 & 0.370 [0.294, 0.533] & 0.239 [0.071, 0.311] & 0.296 & 0.142 & 0.449 & 0.113 \\
 & mean & 25 (7)$^\ast$ & 4 & 0 & 0 & 5 & 16 & 0 & 0.200 [0.100, 0.280] & 0.389 & 0.207 & 0.371 & 0.033 \\
\addlinespace
Gemma-2-2B & zero & 230 (10) & 55 & 67 & 1 & 26 & 81 & 0.296 [0.199, 0.416] & 0.113 [0.077, 0.145] & 0.325 & 0.211 & 0.414 & 0.050 \\
 & mean & 229 (9) & 57 & 0 & 2 & 66 & 104 & $8.73\times10^{-3}$ [0, 0.016] & 0.288 [0.223, 0.385] & 0.400 & 0.277 & 0.303 & 0.020 \\
\addlinespace
Gemma-2-9B & zero & 457 (10) & 146 & 130 & 1 & 74 & 106 & 0.287 [0.189, 0.414] & 0.162 [0.127, 0.208] & 0.326 & 0.277 & 0.370 & 0.027 \\
 & mean & 456 (9) & 142 & 0 & 4 & 149 & 161 & $8.77\times10^{-3}$ [0, 0.017] & 0.327 [0.278, 0.394] & 0.377 & 0.315 & 0.292 & 0.016 \\
\addlinespace
Llama-3.1-8B & zero & 23 (7)$^\ast$ & 2 & 1 & 0 & 3 & 17 & 0.043 [0, 0.125] & 0.130 [0.040, 0.238] & 0.325 & 0.466 & 0.186 & 0.023 \\
 & mean & 23 (7)$^\ast$ & 0 & 0 & 0 & 8 & 15 & 0 & 0.348 [0.261, 0.467] & 0.507 & 0.207 & 0.242 & 0.044 \\
\bottomrule
\end{tabular}
\end{adjustbox}
\caption{The probe against single-latent ablations at $\delta_0=10^{-3}$, with counts of pairs by category, shares of K1 and K2 and of K3 among all pairs and component shares of $|\st|+|\cm_{\rm fin}|+|\cm_{\rm oth}|+|\ec|$, equally weighted over tasks. Because terms can cancel, these shares partition the summed magnitudes rather than the gap. A zero without interval is exact.}\label{tab:x8}
\end{table}

\paragraph{Positional sparsity.}\label{app:position}
SAE latents fire at few positions, whereas the probe reads only the final one. The analysis is post hoc and descriptive with four prespecified outcomes. It relates each outcome to the final-position firing share by a within-task Spearman correlation with a permutation $p$-value (\Cref{app:prespec}). Of the $548$ zero-ablation pairs with $|\alpha|>\delta_0$, the $213$ in K1 never fire at the final position. Their unread share $u\coloneqq|\cm_{\rm oth}|/(|\st|+|\cm_{\rm fin}|+|\cm_{\rm oth}|+|\ec|)$ has median $0.656$, $0.699$ and $0.784$ on GPT-2 and the Gemma models ($15$, $67$ and $130$ pairs), $0.753$ pooled over models and $0.484$ for the one Llama-3.1-8B pair. The rest is the edit-gradient covariance, which also acts only at unread positions because the edit vanishes at the final position.

\subsection{Probes on FDFI's position-averaged summaries}\label{app:fdfi-summary}

FDFI scores a latent through its code averaged over the prompt positions, a summary that \Cref{app:e1r} does not cover. For every FDFI-selected pair two further single-latent probes are fitted on the evaluation strata of at least three prompts, one on the decoder score averaged over the prompt positions (a linear summary to which \Cref{thm:gap} applies) and one on the averaged SAE code (FDFI's own input, a nonlinear summary). Moments are pooled within task with weights proportional to the stratum sizes and the intervention is the zero-ablation at every position after the first. Because this attribution weights prompts equally, the comparisons use a final-position probe with the same weights. The analysis is post hoc and descriptive with prespecified readouts and thresholds. Position averaging largely removes the blindness of the probe (\Cref{tab:fdfi-summary}, \Cref{fig:fdfi-summary-panel}), since the share of pairs whose probe ignores the edit falls from $9$ to $53\%$ for the final-position probe to at most $2\%$. The averaged probes still overstate the attribution, with an equal-task average of within-task median $|\alpha/\rho|$ of $0.056$ to $0.298$, against $0.420$, $2.916$, $0.261$ and $0.224$ for the final-position probe. Their descriptive sign disagreement with the attribution is $16$ to $54\%$, at least as high as that of the final-position probe on the three larger models. The prespecified criterion requires a sign disagreement of at least $0.15$ or a median ratio outside $[1/3,3]$ in at least three models for both summaries. It holds in all four models. The numerical-tolerance criterion is unmet for $1$ of $812$ GPT-2 prompts and $16$ of $720$ Llama-3.1-8B prompts. Under the prespecified amendment, recomputed codes agree in sign with FDFI's inputs on every prompt and the readouts agree across numerical variants of the code computation to within $3.33\times10^{-3}$. On Gemma-2-9B and Llama-3.1-8B the averaged codes are recomputed in float32 from the evaluation-split activations rather than taken from FDFI's float16 inputs.

\begin{table}[h]
\centering\scriptsize
\setlength{\tabcolsep}{4pt}
\begin{tabular}{lcccccccc}
\toprule
& & \multicolumn{3}{c}{Probe-blind share} & \multicolumn{3}{c}{Opposite sign of $\rho$ and $\alpha$} & Median $|\alpha/\rho|$ \\
\cmidrule(lr){3-5}\cmidrule(lr){6-8}\cmidrule(lr){9-9}
Model & Pairs & final & avg.\ decoder & avg.\ code & final & avg.\ decoder & avg.\ code & avg.\ decoder / code \\
\midrule
GPT-2 & 46 & 0.391 & 0 & 0 & 0.414 & 0.403 & 0.336 & 0.294 / 0.167 \\
Gemma-2-2B & 230 & 0.452 & 0.013 & 0 & 0.268 & 0.487 & 0.456 & 0.100 / 0.119 \\
Gemma-2-9B & 457 & 0.534 & 0.020 & 0 & 0.436 & 0.542 & 0.516 & 0.103 / 0.056 \\
Llama-3.1-8B & 23 & 0.087 & 0 & 0 & 0.141 & 0.165 & 0.160 & 0.298 / 0.127 \\
\bottomrule
\end{tabular}
\caption{Probes on FDFI's position-averaged summaries under zero-ablation, against the final-position probe with the same weights (final). Blind shares are over pairs. The descriptive sign shares and the ratio column are equal-task averages of within-task shares and medians. The code summary excludes $4$ and $10$ pairs on the Gemma models whose averaged code has negligible pooled within-stratum variance.}\label{tab:fdfi-summary}
\end{table}

\begin{figure}[t]
\centering
\includegraphics[width=0.65\textwidth]{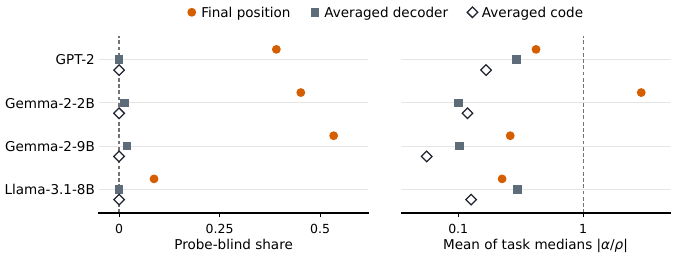}
\caption{Probes on FDFI's position-averaged summaries under zero-ablation, a post hoc descriptive analysis (\Cref{tab:fdfi-summary}). The left panel gives the probe-blind share pooled over pairs and the right panel the equal-task average of within-task median $|\alpha/\rho|$ over pairs with $\rho\ne0$ on a logarithmic scale. Dashed lines mark a zero blind share and equal magnitudes. The Gemma-2-2B final-position ratio of $2.916$ is driven by one task with a single nonzero-probe pair.}\label{fig:fdfi-summary-panel}
\end{figure}

%% file: sections/app_finite.tex
\section{Ablation terms and realized intervention effects}\label{app:finite}

This appendix supports the statements of \Cref{sec:results} that the covariance term in the attribution of an ablation is of the order of its mean-activation term and that gradient-based first-order predictions of realized single-latent ablations agree in sign with the realized effect far more often than probe-based ones. \Cref{app:e8} splits the attribution of each ablation into its two terms and \Cref{app:singletons} compares first-order predictions of single-latent ablations with their realized effects. Every intervention keeps the SAE error term and acts at every position but the first. Zero-ablation sets the latent's code to $0$ and mean-ablation to its selection-split task mean. The ablation terms $L_j$, $M_j$ and $C_j$ of \eqref{eq:ablate} are estimated within strata of at least three prompts and pooled with weights $n_k-1$. Intervals are $95\%$ percentile bootstrap intervals with $2{,}000$ replicates, whose resampling unit each caption states.

\subsection{The mean-activation and covariance terms}\label{app:e8}

\paragraph{Design.} For each selected latent $j$ and task, $L_j$, $M_j$ and $C_j$ are estimated unbiasedly within every stratum. The edit is the one applied in the realized ablations of \Cref{app:singletons}, the latent's code times its unnormalized decoder row and, where the SAE release defines one, its activation scaling factor. Under mean-ablation the code is replaced by its deviation from the reference mean.

\paragraph{Results.} The covariance term is of the same order as the mean term. Under zero-ablation, $|C_j|>|M_j|$ for $52$, $34$, $39$ and $17\%$ of direction-task estimates. Adding $C_j$ reverses the sign of $M_j$ for $26$, $17$, $18$ and $4\%$. The median share $|C_j|/(|M_j|+|C_j|)$ is $0.511$, $0.381$, $0.375$ and $0.267$, with upper deciles of $0.635$ to $0.851$ (\Cref{tab:e8a}). The fraction with $|C_j|>|M_j|$ is stable when strata of two prompts are admitted. Under mean-ablation it is larger on three models and nearly equal on GPT-2, with $C_j$ reversing the sign of $M_j$ for $13$ to $20\%$. The definitions of \Cref{cor:ablate} explain the large shares. At one position with nonzero means and positive variances, $|C_j|/|M_j|$ is the absolute activation-gradient correlation times the coefficients of variation of the code and of the directional derivative. Sparse codes have a large code coefficient of variation. A probe supports at most a mean-activation account, which therefore misses a term of the same order as the one it keeps.

\begin{table}[t]
\centering\footnotesize
\setlength{\tabcolsep}{5pt}
\begin{tabular}{@{}lcccc@{}}
\toprule
& GPT-2 & Gemma-2-2B & Gemma-2-9B & Llama-3.1-8B \\
\midrule
\multicolumn{5}{@{}l}{\emph{Zero-ablation}}\\
Estimates (tasks) & 46 (8) & 230 (10) & 457 (10) & 23 (7)$^\ast$ \\
$|C_j|>|M_j|$ & 0.522 & 0.339 & 0.385 & 0.174 \\
\quad equal task weights & 0.317 & 0.287 & 0.335 & 0.097 \\
\quad strata with $n_k\ge2$ & 0.500 & 0.344 & 0.360 & 0.208 \\
$\operatorname{sign}(M_j+C_j)\ne\operatorname{sign}(M_j)$ & 0.261 & 0.174 & 0.182 & 0.043 \\
Share of $C_j$ [q10, q90] & 0.511 [0.144, 0.848] & 0.381 [0.061, 0.851] & 0.375 [0.048, 0.843] & 0.267 [0.037, 0.635] \\
$C_j\ne0$, BH at $q=0.10$ & 24/46 & 27/227 & 58/447 & 6/23 \\
\midrule
\multicolumn{5}{@{}l}{\emph{Mean-ablation}}\\
Estimates (tasks) & 25 (7)$^\ast$ & 229 (9) & 456 (9) & 23 (7)$^\ast$ \\
$|C_j|>|M_j|$ & 0.520 & 0.424 & 0.388 & 0.217 \\
$\operatorname{sign}(M_j+C_j)\ne\operatorname{sign}(M_j)$ & 0.200 & 0.201 & 0.189 & 0.130 \\
Share of $C_j$ [q10, q90] & 0.507 [0.136, 0.859] & 0.441 [0.054, 0.881] & 0.384 [0.045, 0.851] & 0.281 [0.040, 0.586] \\
\bottomrule
\end{tabular}
\caption{The two terms of \eqref{eq:ablate} per direction-task estimate, in strata of at least three prompts unless stated otherwise. The equal-task-weights row averages the within-task fraction over tasks. The share of $C_j$ is the median of $|C_j|/(|M_j|+|C_j|)$, with its tenth and ninetieth percentiles. BH counts the Benjamini-Hochberg rejections of $C_j=0$ within model among the estimates with a defined standard error, based on stratum-level $t$ intervals (\Cref{app:coverage}).}\label{tab:e8a}
\end{table}

\subsection{Single-latent ablations}\label{app:singletons}

\paragraph{Design.} On up to $200$ evaluation prompts per task, in five folds, each latent of a task's selection is zero-ablated alone at every position but the first. This gives the finite effect $D_{ij}\coloneqq D(\Delta_j(H_i))=F(H_i+\Delta_j(H_i))-F(H_i)$ at the activation $H_i$ of prompt $i$. Three first-order predictions are compared with it on the prompts in strata of at least three prompts. The prompt-specific gradient prediction $P_{ij}\coloneqq\nabla F(H_i)^\top\Delta_j(H_i)$ is attribution patching with the prompt's gradient, with expectation $L_j=M_j+C_j$. The prediction $M_{ij}$ replaces $\nabla F(H_i)$ by its leave-one-out stratum mean, with expectation $M_j$. The probe-based prediction ${\rm probe}_{ij}$ applies a cross-fitted within-stratum slope of the outcome on the final-position decoder score to the final-position code change. Summaries weight tasks equally and intervals resample tasks within model. The prespecified criterion is a better prediction of single-latent ablations by the single-latent attribution than by the probe in at least three models. \Cref{app:prespec} lists it with the other prespecified analyses.

\paragraph{Results.} The criterion holds in all four models, with aggregate gain $0.197$ $[0.101,0.324]$ (\Cref{tab:x45s}). The prompt-specific gradient prediction has the sign of the finite effect for $92$, $91$, $93$ and $97\%$ of prompt-latent rows (equal task weights). The probe-based prediction reaches $21$, $22$, $17$ and $35\%$, since it predicts no change whenever the latent is inactive at the final position (\Cref{cor:ablate}). Adding $C$ improves on the mean-gradient prediction in all four models.

\begin{table}[t]
\centering\footnotesize
\setlength{\tabcolsep}{4pt}
\begin{tabular}{@{}lcccc@{}}
\toprule
Single-latent zero-ablations & GPT-2 & Gemma-2-2B & Gemma-2-9B & Llama-3.1-8B \\
\midrule
Prompt-latent rows (tasks) & 1{,}795 (7) & 10{,}727 (9) & 29{,}180 (9) & 1{,}782 (7) \\
MAE of $P_{ij}$ & 0.026 & 0.097 & 0.082 & $8.52\times10^{-3}$ \\
MAE of $M_{ij}$ & 0.054 & 0.115 & 0.091 & 0.032 \\
MAE of ${\rm probe}_{ij}$ & 0.329 & 0.216 & 0.196 & 0.308 \\
Gain over the probe & 0.303 & 0.119 & 0.114 & 0.300 \\
\quad 95\% CI & [0.037, 0.799] & [0.078, 0.167] & [0.064, 0.171] & [0.114, 0.611] \\
Gain from $C$, over $M_{ij}$ & 0.029 & 0.018 & $8.96\times10^{-3}$ & 0.023 \\
\quad 95\% CI & [0.013, 0.048] & [$8.34\times10^{-3}$, 0.034] & [$5.28\times10^{-3}$, 0.014] & [0.012, 0.041] \\
Sign agreement, $P_{ij}$ / ${\rm probe}_{ij}$ & 0.920 / 0.212 & 0.905 / 0.217 & 0.926 / 0.175 & 0.975 / 0.349 \\
Within-task $r$ with $D_{ij}$, $P_{ij}$ / ${\rm probe}_{ij}$ & 0.942 / 0.237 & 0.597 / 0.197 & 0.569 / $-0.054$ & 0.971 / 0.417 \\
\bottomrule
\end{tabular}
\caption{Single-latent zero-ablations, with the predictions $P_{ij}$, $M_{ij}$ and ${\rm probe}_{ij}$ of \Cref{app:singletons}. MAE weights tasks equally. Gains are differences in MAE, positive when $P_{ij}$ is more accurate. Their intervals resample tasks within model. Sign agreement is over rows with $D_{ij}\ne0$, with a zero prediction counted as a disagreement. Sign agreement averages within-task rates equally over tasks. The correlation $r$ is the mean over tasks of the within-task Pearson correlation.}\label{tab:x45s}
\end{table}

%% file: sections/app_choice.tex
\section{Choosing latents for intervention}\label{app:choice}

This appendix supports the paragraph of \Cref{sec:results} on choosing latents. Both analyses rank or select latents and then compare the choice with realized intervention effects. The first ranks the FDFI-selected latents of the steering extension of \Cref{app:x6}. The second selects sets of latents on training folds and ablates them on held-out prompts.

\subsection{Rankings against steering responses}\label{app:choice-rank}

\paragraph{Design.} The units are the model-task pairs of the steering extension with at least two FDFI-selected latents ($6$, $9$, $9$ and $4$ units, $28$ in total). The analysis counts the units in which the latent with the largest steering response is ranked first by the FDFI statistic $\hat\phi_l/\widehat{\rm se}_l$ of \Cref{app:selection}, by the coefficient magnitude $|\beta|$ or by the attribution magnitude $|\alpha|$. The primary target is the central difference quotient $|T|$, the response per unit displacement. A second target is the finite half-difference $|\varepsilon T|$ of the positive and negative steering interventions, not a one-sided intervention effect, with $\beta$ and $\alpha$ scaled by the same latent-specific step $\varepsilon$. Uniformly random choice within each unit matches the target in $3.31$ units in expectation. This expectation serves as a reference. The analysis is post hoc. Because the attribution and $T$ use the same evaluation prompts, the agreement is descriptive. \Cref{app:choice-sets} gives a held-out comparison.

\paragraph{Results.} The FDFI statistic and coefficient magnitude seldom rank first the latent with the largest quotient magnitude within the selected candidate set (\Cref{tab:choice-rank}). The attribution does so in nearly every unit at small steps. Its advantage holds through one maximum activation and fades at five maximum activations, beyond the range of first-order accuracy (\Cref{app:x6}). With the scaled target the coefficient's ranking gains at large steps, because the step sizes differ across latents. Over units with at least three candidates, the median within-task Spearman correlation between the FDFI statistic and $|T|$ at one standard deviation is $0.151$, $0.026$, $0.054$ and $0.200$.

\begin{table}[h]
\centering\footnotesize
\setlength{\tabcolsep}{5pt}
\begin{tabular}{@{}llccccccc@{}}
\toprule
& & \multicolumn{4}{c}{Within-stratum standard deviations} & \multicolumn{3}{c}{Maximum activations} \\
\cmidrule(lr){3-6}\cmidrule(lr){7-9}
Target & Ranked by & 1 & 2 & 5 & 10 & 1 & 2 & 5 \\
\midrule
$|T|$ & FDFI statistic & 4 & 4 & 3 & 3 & 3 & 4 & 2 \\
& $|\beta|$ & 2 & 2 & 2 & 2 & 1 & 3 & 3 \\
& $|\alpha|$ & 28 & 27 & 19 & 17 & 16 & 10 & 5 \\
\addlinespace
$|\varepsilon T|$ & FDFI statistic & 4 & 4 & 3 & 3 & 5 & 4 & 2 \\
& $|\varepsilon\beta|$ & 7 & 7 & 5 & 4 & 4 & 6 & 8 \\
& $|\varepsilon\alpha|$ & 28 & 26 & 20 & 17 & 18 & 12 & 10 \\
\bottomrule
\end{tabular}
\caption{Number of the $28$ model-task units in which the latent with the largest steering response is ranked first. Uniformly random choice gives $3.31$ in expectation. Post hoc and descriptive.}\label{tab:choice-rank}
\end{table}

\subsection{Held-out set ablations under six selection rules}\label{app:choice-sets}

\paragraph{Design.} The units are the model-task pairs whose FDFI selection holds at least three latents ($5$, $10$, $10$ and $3$ units, $28$ in total). They use the up to $200$ evaluation prompts per task of the set-ablation design in five folds, including prompts excluded by the primary stratum-size restriction and the antonym task. Correlation, rank-test, class-mean-difference and coefficient scores are computed on the four training folds. Attribution sets also use those folds. The FDFI comparator uses the selection-split statistic of \Cref{app:selection}. Every rule selects $k_t=\min(10,|S_t|)$ latents from the active pool. The rules are the absolute correlation between the position-averaged code and the target, the absolute statistic of a Mann-Whitney test between the upper and lower target terciles, the absolute class-mean difference between those terciles, the magnitude $|\beta|$ of the within-stratum pooled final-position coefficient over training strata with at least two prompts, the FDFI selection of \Cref{app:selection} and the attribution set. The attribution set maximizes the training-fold first-order set effect $|\sum_{j\in K}L_j|$ over sets $K$ of $k_t$ active latents. Each set is zero-ablated on the held-out fold at every position but the first, keeping the SAE error term. The primary contrast is $d_c\coloneqq A_{\rm attr}-A_c$, where $A$ is the absolute held-out mean finite effect in standardized units, averaged over folds and then equally over units. Intervals resample strata within model, task and fold. A two-stage version resamples tasks within model and then strata. Both use $2{,}000$ replicates and hold the fitted selections fixed. The comparison is post hoc, specified after the aggregate results and before evaluation (\Cref{app:prespec}).

\paragraph{Results.} Attribution sets produce larger held-out ablation effects than correlation, rank-test and coefficient sets in every unit, with two-stage intervals excluding zero (\Cref{tab:choice-sets}). They exceed class-mean-difference sets in most units, with a two-stage interval that excludes zero at the $95\%$ level but not after a Bonferroni correction over the five comparators. Against the FDFI selection the contrast is positive on three models and in most units, but its two-stage interval includes zero. The two rules are therefore not distinguished under task resampling. Coefficient sets have the smallest edit norms, consistent with latents that rarely fire on held-out prompts. After dividing the absolute held-out mean effect by the root-mean-square edit norm within fold and averaging, the attribution sets still exceed the correlation, rank-test and coefficient sets on every model. The comparison measures the effect on the answer's log-probability. How a larger ablation effect translates into a better intervention depends on the downstream goal.

\begin{table}[h]
\centering\footnotesize
\setlength{\tabcolsep}{4pt}
\resizebox{\textwidth}{!}{%
\begin{tabular}{@{}lcccccccc@{}}
\toprule
Comparator $c$ & $d_c$ & Stratum-resampled 95\% CI & Two-stage 95\% CI & GPT-2 & Gemma-2-2B & Gemma-2-9B & Llama-3.1-8B & Units with $d_c>0$ \\
\midrule
Correlation & 0.949 & [0.903, 1.001] & [0.292, 2.147] & 0.396 & 0.370 & 2.033 & 0.185 & 28/28 \\
Rank test & 0.933 & [0.885, 0.984] & [0.277, 2.121] & 0.415 & 0.276 & 2.044 & 0.278 & 28/28 \\
Class-mean difference & 0.724 & [0.640, 0.798] & [0.061, 1.795] & 0.331 & $-4.91\times10^{-3}$ & 1.795 & 0.236 & 24/28 \\
$|\beta|$ & 1.058 & [1.013, 1.112] & [0.371, 2.304] & 0.454 & 0.506 & 2.137 & 0.309 & 28/28 \\
FDFI & 0.492 & [0.357, 0.602] & [$-0.179$, 1.441] & 0.290 & $-0.209$ & 1.380 & 0.206 & 23/28 \\
\bottomrule
\end{tabular}}
\caption{Held-out set ablations under six selection rules. The contrast $d_c=A_{\rm attr}-A_c$ is in standardized units, equal over units. Model columns give the mean contrast over the model's units. The mean attribution-set effect is $0.467$, $0.511$, $2.165$ and $0.310$ against $0.013$, $5.55\times10^{-3}$, $0.028$ and $7.68\times10^{-4}$ for coefficient sets.}\label{tab:choice-sets}
\end{table}

%% file: sections/app_case.tex
\section{The worked example}\label{app:case}

This appendix gives the selection rule and full results for the worked example in \Cref{fig:overview}A. \Cref{fig:case} shows three held-out prompts. The example is illustrative and post hoc. Its selection rule is specified after the aggregate analyses and independently of the case-level results (\Cref{app:prespec}).

The rule compares two sets of latents per task. Zero-ablating a set $\mathcal S$ of latents removes their codes at every position but the first and keeps the SAE error term. It applies the edit $\Delta_{\mathcal S}\coloneqq\sum_{j\in\mathcal S}\Delta_j$ and has finite effect $D\coloneqq F(H+\Delta_{\mathcal S}(H))-F(H)$ on a prompt. Units are the (model, task) pairs with $k_t=\min(10,|S_t|)\ge3$. On each of five folds $f$ of a task's evaluation prompts, two sets of $k_t$ latents are formed without the prompts of fold $f$ and zero-ablated on them. The FDFI set holds the $k_t$ members of $S_t$ with the largest FDFI statistic. The attribution set maximizes $|\sum_{j\in\mathcal S}L_j^{(-f)}|$ over sets of $k_t$ active latents, where $L_j^{(-f)}$ is the training-fold mean first-order effect of zero-ablating latent $j$, as gradient-based selectors propose \citep{shu2025beyond}. The unit contrast $|\bar D_{\rm attr}|-|\bar D_{\rm FDFI}|$ compares the held-out absolute mean finite effects of the two sets. Among the factual tasks with single-token answers whose two sets differ and whose unit contrast is positive, the rule takes the unit nearest the lower median of the contrast, the first fold in which the two sets differ and the three held-out prompts nearest the quartiles of the unedited log-probability. Gemma-2-9B \texttt{country\_language} is the sole eligible unit (contrast $0.155$). Because eligibility requires a positive contrast, the example serves to illustrate single and set ablations. Evidence on the relative merit of the two selection rules comes from \Cref{app:choice-sets}.

The FDFI representative is the FDFI-set member with the largest FDFI statistic, latent $102779$ (statistic $12.199$, $\beta=3.59\times10^{-3}$, training-fold $L_j=4.40\times10^{-3}$). The attribution representative is the attribution-set member outside the FDFI set with the largest training-fold $|L_j^{(-f)}|$, latent $99439$ ($6.498$, $\beta=-0.091$, $L_j=0.088$). Each prompt is evaluated unedited and under the four zero-ablations of \Cref{tab:case-detail}, none of which changes the top-ranked answer (\Cref{fig:case}).

\begin{figure}[t]
\centering
\includegraphics[width=\textwidth]{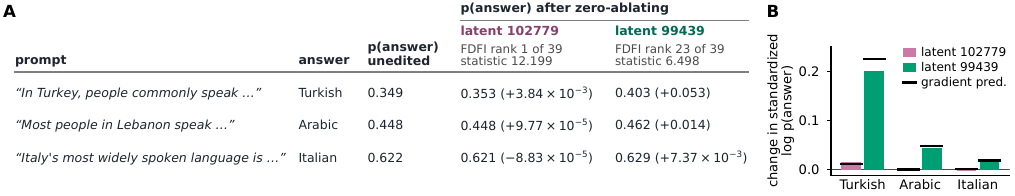}
\caption{The worked example (Gemma-2-9B, country to language). (A) Three held-out prompts chosen by baseline log-probability quantiles, with $p(\text{answer})$ unedited and after zero-ablating latent 102779 or 99439 at every position but the first (change in parentheses). Headers give each latent's FDFI rank and statistic. (B) Realized change of each ablation (bars) against the prompt-specific gradient prediction (ticks), in standardized units. The example is illustrative, selected post hoc as the sole eligible unit of the deterministic rule of this section.}\label{fig:case}
\end{figure}

\begin{table}[t]
\centering\footnotesize
\setlength{\tabcolsep}{5pt}
\begin{tabular}{@{}llcccc@{}}
\toprule
Prompt & & & Realized & First-order & Edit \\
answer (unedited $p$) & Zero-ablated & $p(\text{answer})$ & change & prediction & norm \\
\midrule
\textit{In Turkey, people commonly speak} & FDFI repr. & 0.353 & 0.015 & 0.012 & 97.006 \\
Turkish (0.349) & attribution repr. & 0.403 & 0.201 & 0.224 & 24.246 \\
& FDFI set & 0.334 & $-0.062$ & $-0.043$ & 120.406 \\
& attribution set & 0.457 & 0.381 & 0.479 & 85.688 \\
\addlinespace
\textit{Most people in Lebanon speak} & FDFI repr. & 0.448 & $3.04\times10^{-4}$ & $3.08\times10^{-4}$ & 18.565 \\
Arabic (0.448) & attribution repr. & 0.462 & 0.045 & 0.049 & 15.330 \\
& FDFI set & 0.446 & $-6.49\times10^{-3}$ & $-8.02\times10^{-3}$ & 75.831 \\
& attribution set & 0.511 & 0.188 & 0.187 & 88.173 \\
\addlinespace
\textit{Italy's most widely spoken language is} & FDFI repr. & 0.621 & $-2.03\times10^{-4}$ & $6.66\times10^{-4}$ & 64.386 \\
Italian (0.622) & attribution repr. & 0.629 & 0.017 & 0.019 & 17.907 \\
& FDFI set & 0.562 & $-0.142$ & 0.176 & 110.702 \\
& attribution set & 0.593 & $-0.066$ & 0.316 & 91.278 \\
\bottomrule
\end{tabular}
\caption{The worked example of \Cref{fig:case} (Gemma-2-9B \texttt{country\_language}, first fold). Changes and the prompt-specific gradient predictions, summed over the ten latents for set ablations, are in standardized log-probability units. The edit norm is $(\sum_\tau\|\Delta h_\tau\|^2)^{1/2}$ and repr.\ is the representative latent.}\label{tab:case-detail}
\end{table}

In \Cref{tab:case-detail} the FDFI representative has small realized effects close to its first-order predictions, although its edit is the larger of the two representatives' on every prompt. The attribution detects this combination of a large edit and a near-zero first-order effect, which neither the FDFI statistic nor $\beta$ measures. The attribution representative changes the answer by almost exactly its first-order prediction on every prompt. Because the coefficient and the ablation effect have different edit orientations, their opposite signs alone do not establish a disagreement. On the first two prompts the set ablations agree in sign and roughly in size with their summed first-order predictions. On the Italy prompt both ten-latent sets lower the probability ($-0.142$ and $-0.066$) although their summed predictions are positive ($0.176$ and $0.316$). There the two representatives stay close to their predictions while the set edits are five to six times larger in norm than the attribution representative's. A set ablation can thus depart from its summed first-order prediction on a prompt where single-latent ablations stay close to theirs.